\documentclass[12pt]{article}
\usepackage[utf8]{inputenc}
\usepackage[sectionbib]{natbib}

\usepackage[T1]{fontenc}    

\title{A New Measure of Model Redundancy for Compressed Convolutional Neural Networks}
\author{Feiqing Huang$^\dagger$, Yuefeng Si$^\dagger$, Yao Zheng$^\ddagger$, Guodong Li$^\dagger$\\
	\textit{$\dagger$ Department of Statistics and Actuarial Science, University of Hong Kong}\\
	\textit{$\ddagger$ Department of Statistics, University of Connecticut}}

\usepackage[margin=0.85in]{geometry}
\usepackage{graphicx}
\usepackage{amsfonts}
\usepackage{amsmath,amssymb,amsthm}
\usepackage[mathscr]{euscript}
\usepackage{multirow}
\usepackage{multicol}
\usepackage{amsfonts}
\usepackage{amsmath}
\usepackage{graphicx}
\usepackage{amsthm}
\usepackage{mathtools}
\usepackage{sectsty}
\usepackage{makecell}
\usepackage{amsmath}
\usepackage{amssymb}
\usepackage[toc,page]{appendix}
\usepackage{underscore} 
\usepackage{url}            
\usepackage{booktabs}       
\usepackage{amsfonts}       
\usepackage{microtype}      
\usepackage{xcolor}         
\usepackage{tikz}			
\usetikzlibrary{arrows.meta,
	chains,
	positioning,
	shapes.geometric
}

\newtheorem{assumption}{Assumption}

\newtheorem{lemma}{Lemma}

\newtheorem{theorem}{Theorem}
\newtheorem{corollary}{Corollary}

\newcolumntype{P}[1]{>{\centering\arraybackslash}p{#1}} %

\newcommand{\bm}{\boldsymbol}
\newcommand{\cm}[1]{\mbox{\boldmath$\mathscr{#1}$}}

\newcommand{\ts}{\textsuperscript}
\newcommand{\norm}[1]{\left\lVert#1\right\rVert}

\DeclareMathOperator*{\vectorize}{vec}

\DeclareMathOperator*{\trace}{tr}
\DeclareMathOperator*{\argmin}{arg\,min}

\begin{document}
\setlength{\parindent}{0pt}
\maketitle
\begin{abstract}
	While recently  many designs have been proposed to improve the model efficiency of convolutional neural networks (CNNs) on a fixed resource budget, theoretical understanding of these designs is still conspicuously lacking.
	This paper aims to provide a new framework for answering the question: Is there still any remaining model redundancy in a compressed CNN? We begin by developing a general statistical formulation of CNNs and compressed CNNs via the tensor decomposition, such that the weights across layers can be summarized into a single tensor.  Then, through a rigorous sample complexity analysis, we reveal an important discrepancy between the derived sample complexity and the naive parameter counting, which serves as a direct indicator of the model redundancy.  Motivated by this finding,  we  introduce a new model redundancy measure for compressed CNNs, called the $K/R$ ratio, which further allows for nonlinear activations. The usefulness of this new measure is supported by ablation studies on  popular block designs and datasets.
\end{abstract}

\section{Introduction}

The introduction of AlexNet \cite{krizhevsky2012imagenet}  spurred a line of research in $2$D CNNs, which has progressively achieved high levels of accuracy in the domain of image recognition \cite{simonyan2014very,szegedy2015going,he2016deep,huang2017densely}. The current state-of-the-art CNNs leave little room to achieve significant improvement on accuracy in learning still-images, and attention has hence been diverted towards two directions. The first is to deploy deep CNNs on mobile devices by removing redundancy from the over-parametrized network;  some representative models include MobileNetV1 \& V2, ShuffleNetV1 \& V2 \citep{howard2017mobilenets,sandler2018mobilenetv2,zhang2018shufflenet,ma2018shufflenet}. 
The second direction is to utilize CNNs to learn from higher-order inputs, for instance, video clips \citep{tran2018closer,hara2017learning} and electronic health records \citep{cheng2016risk,suo2017personalized}; this area has not yet seen a widely-accepted state-of-the-art network.
High-order kernel tensors are usually required to account for the multiway dependence of the input. However,
this notoriously leads to heavy computational burden, as the number of parameters to be trained grows exponentially with the dimension of inputs. 
In short, for both directions, model compression is the critical juncture for the success of training and deployment of CNNs.

\paragraph{Compressing CNNs via tensor decomposition}
\citet{denil2013predicting} showed that there is huge redundancy in network weights, as the entire network can be approximately recovered with a small fraction of parameters. Tensor decomposition has recently been widely used to compress the weights in CNNs \citep{lebedev2014speeding,kim2015compression, kossaifi2019factorized,hayashi2019exploring}. 
Specifically, the weights at individual layers are first rearranged into tensors, and then tensor decomposition, CP or Tucker, can be applied separately at each layer to reduce the number of parameters.
Different tensor decompositions for convolution layers lead to a variety of compressed CNN block designs.
For instance, the bottleneck block in ResNet \citep{he2016deep} corresponds to the convolution kernel with a special Tucker low-rank structure. The depthwise separable block in MobileNetV1 and ShuffleNetV2 and the inverted residual block in MobileNetV2 correspond to the convolution kernel with special CP forms.	
Some typical examples of CNN block designs are shown in Figure \ref{fig:bottleneck}; a detailed discussion is given in Section 4.
All the above studies concern 2D CNNs, while \citet{kossaifi2019factorized} and \citet{su2018tensorial} consider tensor decomposition to factorize convolution kernels for higher-order tensor inputs.

\begin{figure}[ht]
	\centering
	\includegraphics[width=0.9\linewidth]{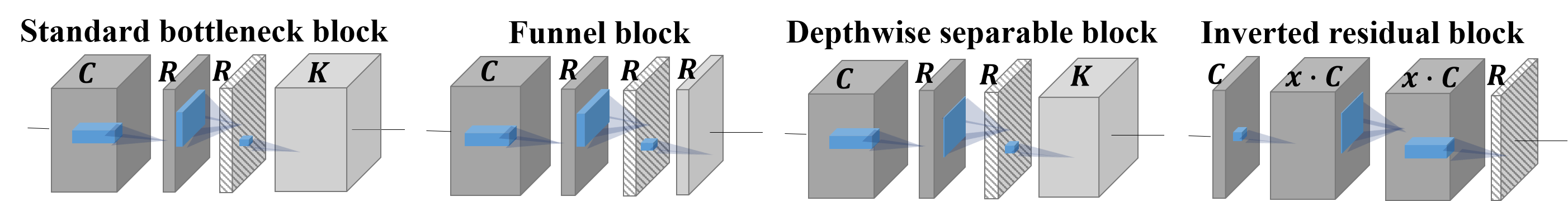}
	\caption{We use thickness of each layer to indicate its relative number of channels. The last 
		layer indicates the output of the CNN block. The dashed layer represents the "bottleneck" of the block.}
	\label{fig:bottleneck}
\end{figure}

Parameter efficiency of the aforementioned architectures has been heuristically justified by methods such as FLOPs counting, naive parameter counting and/or empirical running time.
However,  the theoretical mechanism through which the tensor decomposition compresses CNNs is still largely unknown, and so is  the question  of how to measure the degree of potential model redundancy.
This paper aims to provide a solution from a statistical perspective.
To begin with, we need to clearly specify the definition of model redundancy.

\paragraph{Sample complexity analysis and model redundancy} 
\citet{du2018many} first characterized the statistical sample complexity of a CNN; see also \citet{wang2019compact} for compact autoregressive nets.
Specifically, consider a CNN model, $y=F_{\textrm{CNN}}(\bm{x}, \cm{W})+\xi$, where $y$ and $\bm{x}$ are the output and input, respectively, $\cm{W}$ is a composite weight tensor and $\xi$ is an additive error.
Given the trained and true underlying networks $F_{\textrm{CNN}}(\bm{x}, \cm{\widehat{W}})$ and $F_{\textrm{CNN}}(\bm{x}, \cm{W}^*)$, the root-mean-square prediction error is defined as
\begin{equation}\label{predict}
	\cm{E}(\cm{\widehat{W}})=\sqrt{E_{\bm{x}}|F_{\textrm{CNN}}(\bm{x}, \cm{\widehat{W}})- F_{\textrm{CNN}}(\bm{x}, \cm{W}^*)|^2},
\end{equation}
where $\cm{\widehat{W}}$ and $\cm{W}^*$ are trained and true underlying weights, respectively, and $E_{\bm{x}}$ is the expectation on $\bm{x}$.
The sample complexity analysis, which determines how many samples are needed to guarantee a given tolerance on the prediction error,  provides a theoretical foundation for measuring model redundancy.

Specifically, given  the true underlying model $F_0$, consider two training models $F_1$ and $F_2$ such that $F_0\subseteq F_1\subset F_2$, i.e. $F_1$ is more compressed than $F_2$.
If $F_1$ and $F_2$ achieve the same prediction error on any given training sets, we can then argue that $F_2$ has some redundant parameters compared with $F_1$. Enabled by recent advances in non-asymptotic high-dimensional statistics, this paper provides a rigorous statistical approach to measuring model redundancy. Note that existing theoretical tools are applicable mainly to linear models, or at least is reliant on some form of convexity of the loss function; see \cite{wainwright2019high} for a comprehensive review.

\begin{figure}[h]
	\centering
	\resizebox{0.9\columnwidth}{!}{%
		\begin{tikzpicture}[thick, scale=0.6,
			arrow/.style={-{Stealth[]}}]
			\node[draw,
			rectangle,
			minimum width=0.8cm,
			minimum height=0.8cm, align=center, font=\linespread{0.8}\selectfont] (n1) {Compressed CNN \\ formulation};
			
			\node[draw,
			right=of n1,
			minimum width=0.8cm,
			minimum height=0.8cm, align=center, font=\linespread{0.8}\selectfont
			] (n2) { Sample complexity \\ \small{$d_{\mathcal{M}}^c=R\prod_{i=1}^{N}R_i+\sum_{i=1}^{N}l_iR_i+{\color{blue}RP}$} };
			
			\node[draw,
			below=of n2,
			minimum width=0.3cm,
			minimum height=0.3cm, align=center, font=\linespread{0.8}\selectfont
			] (n3) {Discrepancy};
			
			\node[draw,
			below=of n3,
			minimum width=0.8cm,
			minimum height=0.8cm, align=center, font=\linespread{0.8}\selectfont
			] (n4) {Naive parameter count \\ \small{$d_{\#}^c=R\prod_{i=1}^{N}R_i+\sum_{i=1}^{N}l_iR_i+{\color{blue}RK+KP}$}};
			
			\node[draw,
			right=of n3,
			minimum width=0.8cm,
			minimum height=0.8cm, align=center, font=\linespread{0.8}\selectfont
			] (n5) {$K/R$};
			
			\node[coordinate, 
			right=1cm of n5] (n6) {};
			
			\node[draw,
			above right=0.7cm and 3.7cm of n6,
			minimum width=0.8cm,
			minimum height=0.8cm, align=center, font=\linespread{0.8}\selectfont
			] (n7) {Have redundancy};
			
			\node[draw,
			right=5cm of n6,
			minimum width=0.8cm,
			minimum height=0.8cm, align=center, font=\linespread{0.8}\selectfont
			] (n8) {Possible redundancy};
			
			\node[draw,
			below right=0.7cm and 3.7cm of n6,
			minimum width=0.8cm,
			minimum height=0.8cm, align=center, font=\linespread{0.8}\selectfont
			] (n9) {No redundancy};
			
			\draw[-latex] (n1) edge (n2)
			(n2) edge (n3)
			(n4) edge (n3)
			(n3) edge (n5);
			
			\draw (n5) -- (n6);
			
			\draw[-latex] (n6) |- (n7)
			node[pos=0.8,above]{$K/R\gg 1$};
			
			\draw[-latex] (n6) -- (n8)
			node[midway,above]{$K/R$ slighly larger than 1};
			
			\draw[-latex] (n6) |- (n9)
			node[pos=0.8,above]{$K/R=1$};
			
		\end{tikzpicture}
	}
	\caption{Proposed general framework for model redundancy quantification, where $K= $ \# of channels in the output layer, and $R= $ \# of channels in the bottleneck layer; see Sections 2 \& 3 for detailed definitions of notations.}
	\label{fig:flowchart}
\end{figure}
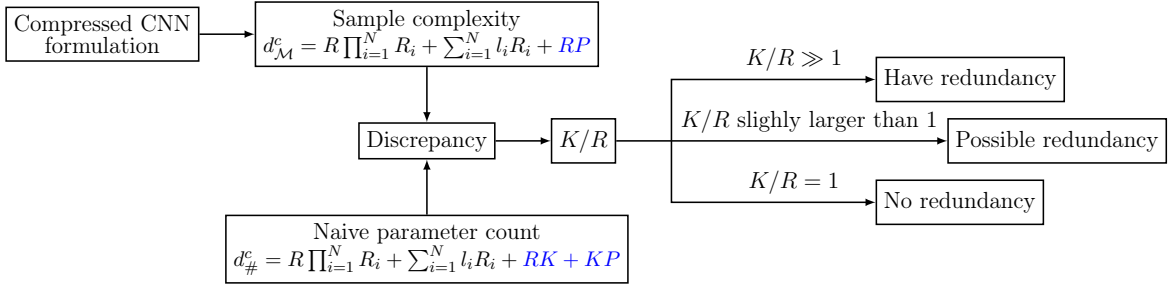

\paragraph{A general framework and a new  measure of model redundancy} 
For  compressed CNNs, we introduce a general sample complexity based framework and a new measure to quantify the model redundancy; see Figure \ref{fig:flowchart} for an illustration. Specifically:

1. For high-order inputs, we first develop statistical formulations for CNNs and compressed CNNs via tensor decomposition, which exactly replicates the operations in a multilayer CNN.

2. Then based on the sample complexity analysis of compressed CNNs with linear activations, we  discover an important discrepancy between the derived sample complexity ($d_{\mathcal{M}}^c$) and the naive parameter counting ($d_{\#}^c$), an indicator of potential model redundancy.

3. By further taking into account the nonlinearity in both models and data, we introduce a new measure, called the $K/R$ ratio, to quantify any inefficient redundancy in a CNN block design.

It is  worth noting that the above framework is not limited to CNNs, but is potentially extensible to other network models: RNNs, GNNs, etc.

\subsection{Comparison with other existing works}
Since our theory is based on
upper bounding the prediction error in \eqref{predict}, it is necessary to point out how it differs from the vast literature on the generalization ability of deep neural networks \citep{arora2018stronger,li2020understanding}. Studies on generalization bounds focus on how well a network generalizes to new test samples, whereas we aim to quantify the remaining redundancy in a compressed network architecture.  
The distinctive  objectives naturally leads to different methodologies. Note that the discrepancy between the sample complexity and the naive parameter counting gives a direct measure of model redundancy, while our approach is the only one that enables an exact analysis of this quantity.

A comprehensive review by \citet{valle2020generalization} shows that, most types of generalization bounds are characterized by the norm of network weights or other related quantities such as the Lipschitz constant of the network; some examples include the margin-based bounds \citep{bartlett2017spectrally}, the sensitivity-based bounds \citep{neyshabur2017pac}, the NTK-based bounds \citep{cao2019tight} or the compression-based bounds \citep{li2020understanding}.
They are model-agnostic in the sense that, only a generic functional form of the network is required and the weights are then implicitly regularized by some learning algorithms. As a result, their derived sample complexity cannot be staightforwardly used to pinpoint the redundancy in a well-specified network architecture. The VC bounds \citep{vapnik2013nature}, though not norm-based per se, are easily corrupted by redundant parameters, and hence are also unsuitable for our target.

Other existing works on  theoretical understanding of neural networks include 
parameter recovery with gradient-based algorithms for deep neural networks \citep{fu2020guaranteed}; the development of other provably efficient algorithms \citep{,du2018improved}; and the investigation of convergence in an over-parameterized regime \citep{allen2018convergence}. Our work also differs greatly from these in both aim and methodology, and we do not consider computational complexity or algorithmic convergence. 

\subsection{Notations}
We follow the notations in \citet{kolda2009tensor} to denote vectors by lowercase boldface letters, e.g. $\bm{a}$; matrices by capital boldface letters, e.g. $\bm{A}$; tensors of order 3 or higher by Euler script boldface letters, e.g. $\cm{A}$. 
For an $N$th-order tensor $\cm{A}\in\mathbb{R}^{l_1\times\cdots\times l_N}$, denote its elements by $\cm{A}(i_1,i_2,\dots,i_N)$ and the $n$-mode unfolding by $\cm{A}_{(n)}$, where the columns of $\cm{A}_{(n)}$ are the $n$-mode vectors of $\cm{A}$, for $1\leq n\leq N$.
The vectorization of the tensor $\cm{A}$ is denoted by $\text{vec}(\cm{A})$. 
The inner product of two tensors $\cm{A},\cm{B}\in\mathbb{R}^{l_1\times\cdots\times l_N}$ is defined as $\langle\cm{A},\cm{B}\rangle=\sum_{i_1}\cdots\sum_{i_N}\cm{A}(i_1,\dots,i_N)\cm{B}(i_1,\dots,i_N)$, and the Frobenius norm is $\|\cm{A}\|_{\text{F}}=\sqrt{\langle\cm{A},\cm{A}\rangle}$.
The mode-$n$ multiplication $\times_n$ of a tensor $\cm{A}\in\mathbb{R}^{l_1\times\cdots\times l_N}$ and a matrix $\bm{B}\in\mathbb{R}^{p_n\times l_n}$ results in a tensor $\cm{C}$ of size $\mathbb{R}^{l_1\times\cdots\times p_n\times\cdots\times l_N}$, where
$\cm{C}(i_1,\dots,j_n,\dots,i_N)
=\sum_{i_n=1}^{l_n}\cm{A}(i_1,\dots,i_n,\dots,i_N)\bm{B}(j_n,i_n)$,
for $1\leq j_n\leq p_n$ and $1\leq n\leq N$. 
We use the symbol ``$\otimes$" to denote the Kronecker product.
For any positive sequences $\{a_n\}$ and $\{b_n\}$, $a_n\lesssim b_n$ and $a_n\gtrsim b_n$ denote that there exists a positive constant $C$ such that $a_n\leq Cb_n$ and $a_n\geq Cb_n$, respectively.

There are two commonly used methods for tensor decomposition. The first one is Canonical Polyadic (CP) decomposition \citep{kolda2009tensor}: it factorizes the tensor $\cm{A}\in\mathbb{R}^{l_1\times\cdots\times l_N}$ into a sum of rank-1 tensors, i.e.
$\cm{A} = \sum_{r=1}^{R}\alpha_r\bm{h}^{(1)}_r\circ \bm{h}^{(2)}_r\circ\cdots\circ\bm{h}^{(N)}_r$, where $\bm{h}^{(j)}_r$ is a unit-norm vector of size $\mathbb{R}^{l_j}$ for all $1\leq j\leq N$.
The CP rank, denoted by $R$, is the smallest number of rank-1 tensors.
The other one is Tucker decomposition: the Tucker ranks of an $N$th-order tensor $\cm{A}\in\mathbb{R}^{l_1\times\cdots\times l_N}$ are defined as the matrix ranks of the unfoldings of $\cm{A}$ along all modes, i.e. $\text{rank}_i(\cm{A})=\text{rank}(\cm{A}_{(i)})$, $1\leq i\leq N$. 
If the Tucker ranks of $\cm{A}$ are $(R_1,\dots,R_N)$, 
then there exist a core tensor $\cm{G}\in\mathbb{R}^{R_1\times\cdots\times R_N}$ and matrices $\bm{H}^{(i)}\in\mathbb{R}^{l_i\times R_i}$ such that
$\cm{A}=\cm{G}\times_1\bm{H}^{(1)}\times_2\bm{H}^{(2)}\cdots\times_N\bm{H}^{(N)}$, 
known as Tucker decomposition \citep{tucker1966some}.

\section{Statistical formulation for (compressed) CNNs}

\subsection{Model for basic three-layer CNNs}\label{ses:3-layer}

Consider a three-layer CNN with one convolution, one average pooling and one fully-connected layer. Specifically, for a general tensor-structured input $\cm{X}\in\mathbb{R}^{d_1\times d_2\times\cdots\times d_N}$, we first perform its convolution with an $N$th-order kernel tensor $\cm{A}\in\mathbb{R}^{l_1\times l_2\times\cdots\times l_N}$ to get an intermediate output $\cm{X}_c\in\mathbb{R}^{m_1\times m_2\times\cdots\times m_N}$, and then use average pooling with pooling sizes $(q_1,\cdots,q_N)$ to obtain another intermediate output $\cm{X}_{cp}\in\mathbb{R}^{p_1\times p_2\times\cdots\times p_N}$. Finally, $\cm{X}_{cp}$ goes through a fully-connected layer, with weight tensor $\cm{B}\in\mathbb{R}^{p_1\times p_2\times\cdots\times p_N}$ to produce a scalar output; see Figure \ref{fig:CNN}. 
\begin{figure}[ht]
	\centering
	\includegraphics[width=0.98\linewidth]{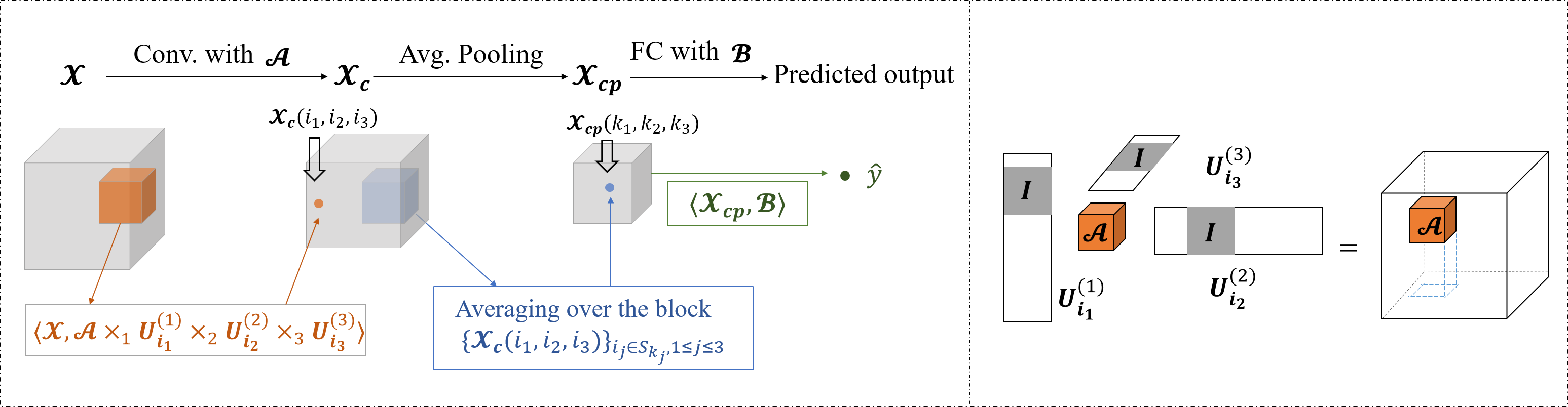}
	\caption{When $N=3$, a 3-layer CNN for an input $\cm{X}$ with one kernel tensor $\cm{A}$, average pooling, and fully-connected weights $\cm{B}$. The right panel shows that the matrices $\bm{U}_{i_j}^{(j)}$ act as positioning factors to stretch the kernel tensor $\cm{A}$ into the same size as $\cm{X}$. The white spaces indicate zero entries.}
	\label{fig:CNN}
\end{figure}

We first consider the convolution layer with  stride size  $s_c$ along each dimension. Assume that $m_j=(d_j-l_j)/s_c+1$ are integers for $1\leq j\leq N$; otherwise zero-padding will be needed. For $1\leq i_j\leq m_j$, $1\leq j\leq N$, let
\begin{equation}\label{eq:U}
	\bm{U}_{i_j}^{(j)} = [\underbrace{\bm{0}}_{(i_j-1)s_c} \underbrace{\bm{I}}_{l_j} \underbrace{\bm{0}}_{d_j-(i_j-1)s_c-l_j}]^\prime \in\mathbb{R}^{d_j\times l_j},
\end{equation}
which acts as positioning factors to stretch the kernel tensor $\cm{A}$ 
into a tensor of the same size as $\cm{X}$, while the rest of entries are filled with zeros; see Figure \ref{fig:CNN}. As a result, $\cm{X}_c$ has entries
$
\cm{X}_c(i_1,\ldots,i_N)= g\left(\langle\cm{X},\cm{A}\times_1\bm{U}_{i_1}^{(1)}\times_2\cdots \times_N\bm{U}_{i_N}^{(N)} \rangle\right),
$
where $g(\cdot)$ is an activation function.

For the pooling layer with  stride size  $s_p$ along each dimension, we assume the pooling sizes $\{q_j\}_{j=1}^N$ satisfy $m_j=q_j + (p_j-1)s_p$, where the sliding windows can be overlapped. But for ease of notation, we can simply take $q_j = m_j/p_j$.
There are a total of $P=p_1\cdots p_N$ pooling blocks in $\cm{X}_c$, with the $(k_1,\cdots,k_N)$th block formed by $\{\cm{X}_c(i_1,\cdots,i_N)\}_{i_j\in S_{k_j},1\leq j\leq N}$, where the index set
$
S_{k_j} = \{(k_j-1)q_j+1\leq i\leq k_jq_j:1\leq j\leq N\}.
$
By taking the average per block, the resulting tensor $\cm{X}_{cp}$ has entries of
$\cm{X}_{cp}(k_1,\ldots,k_N)=(q_1\cdots q_N)^{-1}\sum_{i_j\in S_{k_j},1\leq j\leq N}\cm{X}_c(i_1,\ldots,i_N)$, and the predicted output has the form of 
$\widehat{y}= \langle\cm{X}_{cp},\cm{B}\rangle$ .

Similarly, for a CNN with $K$ kernels, denote by $\{\cm{A}_k,\cm{B}_k\}_{k=1}^K$  the set of kernels and fully-connected weights, where $\cm{B}_k\in\mathbb{R}^{p_1\times p_2\times\cdots\times p_N}$ and 
$\cm{A}_k\in\mathbb{R}^{l_1\times l_2\times\cdots\times l_N}$.
Further taking into account an additive random error  $\xi^i$, the CNN model can be mathematically formulated into
\begin{equation}
	\label{eq:model-nonlinear}
	y^i = \widehat{y}^i+\xi^i = \sum_{k=1}^K\langle\cm{X}_{cp}(\cm{X}_i,\cm{A}_k),\cm{B}_k\rangle + \xi^i, \hspace{2mm}1\leq i\leq n,
\end{equation}
where the notation of $\cm{X}_{cp}(\cm{X}_i,\cm{A}_k)$ is used to emphasize its dependence on $\cm{X}_i$ and $\cm{A}_k$.

When the activation function  $g(\cdot)$ is linear, the entries of $\cm{X}_{cp}$ have a simplified form of 
$
\cm{X}_{cp}(k_1,\ldots,k_N) = \langle\cm{X},\cm{A}\times_1\bm{U}^{(1)}_{\mathcal{F},k_1}\times_2\cdots \times_N\bm{U}^{(N)}_{\mathcal{F},k_N}\rangle,
$
where $\bm{U}^{(j)}_{\mathcal{F},k_j} = q_j^{-1}\sum_{i_j\in S_{k_j}}\bm{U}_{i_j}^{(j)}$.
The intuition behind  is that, the averaging operation on subtensors of $\cm{X}_c$ can be transfered to the positioning matrices $\{\bm{U}_{i_j}^{(j)}\}_{i_j\in S_{k_j},1\leq j\leq N}$.
Denote $\bm{U}^{(j)}_{\mathcal{F}} = (\bm{U}^{(j)}_{\mathcal{F},1},\cdots,\bm{U}^{(j)}_{\mathcal{F},p_j})$ for $1\leq j\leq N$, and then
$
\langle\cm{X}_{cp},\cm{B}\rangle =\langle\cm{X},(\cm{B}\otimes\cm{A})\times_1\bm{U}^{(1)}_{\mathcal{F}}\times_2\cdots \times_N\bm{U}^{(N)}_{\mathcal{F}}\rangle.
$
As a result, the CNN model in \eqref{eq:model-nonlinear} has the linear form of $y^i = \langle\cm{X}^i, \cm{W}_X\rangle + \xi^i$, where  $\cm{X}^i\in\mathbb{R}^{d_1\times d_2\times\cdots\times d_N}$, and the composite weight tensor is 
\begin{equation}\label{tucker-coe}
	\cm{W}_X = (\sum_{k=1}^{K}\cm{B}_k\otimes\cm{A}_k)\times_1\bm{U}^{(1)}_\mathcal{F}\times_2\cdots\times_N\bm{U}^{(N)}_\mathcal{F}.
\end{equation}

Let $P = \prod_{i=1}^{N}p_i$ and $L=\prod_{i=1}^{N}l_i$.  Note that the naive count of the number of parameters in model \eqref{eq:model-nonlinear} is given by 
\begin{equation}\label{firstparamno}
	d_{\#}^u = K(P+L),
\end{equation}
where the superscript $u$ stands for the uncompressed CNN.
Other methods for measuring model efficiency include FLOPs counting and empirical running time, yet they are all heuristic without theoretical justification.

The  weight tensor in \eqref{tucker-coe} has a compact ``Tucker-like" structure characterizing the weight-sharing pattern of the CNN.
The factor matrices $\{\bm{U}^{(j)}_\mathcal{F}\in\mathbb{R}^{d_j\times l_jp_j}\}_{j=1}^N$ are fixed and solely determined by CNN operations on the inputs. They have full column ranks since $p_jl_j\leq d_j$ always holds. The core tensor is a special Kronecker product that depicts the layer-wise interaction between 
weights.

\subsection{Model for compressed CNNs via tensor decomposition} 

For high-order inputs, a deep CNN with a large number of kernels may involve heavy computation, which renders it difficult to train on portable devices with limited resources.
In real applications, many compressed CNN block designs have been proposed to improve the efficiency,  most of which are based on either matrix factorization or tensor decomposition \citep{lebedev2014speeding,kim2015compression,astrid2017cp,kossaifi2019factorized}.

Tucker decomposition can be used to compress CNNs. This is equivalent to introducing a multilayer CNN block; see Figure 3 in \citet{kim2015compression}.
Specifically, we stack the kernels $\cm{A}_k\in \mathbb{R}^{l_1\times l_2\times\cdots\times l_{N}}$ with $1\leq k\leq K$ into a higher order tensor $\cm{A}_\mathrm{stack}\in\mathbb{R}^{l_1\times l_2\times\cdots\times l_{N}\times K}$. The $(N+1)$th dimension is commonly referred to in the literature as the output channels. Assume that $\cm{A}_\mathrm{stack}$ has Tucker ranks of $(R_1,R_2,\cdots,R_N,R)$. Then we can write
\begin{equation}\label{Tucker}
	\cm{A}_\mathrm{stack} = \cm{G}\times_1 \bm{H}^{(1)}\times_2 \bm{H}^{(2)}\cdots\times_{N+1} \bm{H}^{(N+1)},
\end{equation}
where $\cm{G}\in \mathbb{R}^{R_1\times\cdots\times R_N\times R}$ is the core tensor, and $\bm{H}^{(j)}\in \mathbb{R}^{l_j\times R_j}$ for $ 1\leq j\leq N$ and $\bm{H}^{(N+1)}\in\mathbb{R}^{K\times R}$ are factor matrices. The naive count of the number of parameters in this model is
\begin{equation}\label{secondparamno}
	d_{\#}^c = R\prod_{i=1}^{N}R_i+\sum_{i=1}^{N}l_iR_i+RK+KP.
\end{equation}
The CP decomposition is also popular in compressing CNNs \citep{astrid2017cp,kossaifi2019factorized,lebedev2014speeding}, in which case the naive parameter counting becomes $R(\sum_{i=1}^{N}l_i + K + 1) + KP$, where $R$ is the CP rank.

\section{Sample complexity analysis}

We begin with the sample complexity analysis of the CNN model in \eqref{eq:model-nonlinear} with linear activation functions. Let ${\cm{Z}}^i = \cm{X}^i\times_1\bm{U}^{(1)\prime}_\mathcal{F}\times_2\bm{U}^{(2)\prime}_\mathcal{F}\times_3\cdots\times_N\bm{U}^{(N)\prime}_\mathcal{F}\in\mathbb{R}^{l_1p_1\times l_2p_2\times\cdots\times l_Np_N}$. By the derivations in Section 2.1, the CNN model can be written equivalently as 
\begin{equation}\label{eq:model-x}
	y^i = \langle\cm{Z}^i, \cm{W}\rangle + \xi^i=\sum_{k=1}^{K}\langle\cm{Z}^i, \cm{B}_k\otimes\cm{A}_k\rangle + \xi^i.
\end{equation}
where $\cm{W}=\sum_{k=1}^{K}\cm{{B}}_k\otimes\cm{{A}}_k$. 
Then  $\cm{W}^*=\sum_{k=1}^{K}\cm{B}_k^*\otimes\cm{A}_k^*$ is the true weight, and the trained weight is given by $\cm{\widehat{W}}=\sum_{k=1}^{K}\cm{\widehat{B}}_k\otimes\cm{\widehat{A}}_k$, where
\begin{equation}\label{lse}
	\{\cm{\widehat{B}}_k,\cm{\widehat{A}}_k\} _{1\leq k\leq K}
	= \argmin_{\cm{B}_k,\cm{A}_k,1\leq k\leq K}\frac{1}{n}\sum_{i=1}^{n}\left( y^i -  \sum_{k=1}^{K}\langle\cm{Z}^i, \cm{B}_k\otimes\cm{A}_k\rangle\right)^2.
\end{equation}

Denote $\bm{U}_G = \bm{U}_\mathcal{F}^{(1)}\otimes\{\bm{U}_\mathcal{F}^{(N)}\otimes[\bm{U}_\mathcal{F}^{(N-1)}\otimes \cdots \otimes (\bm{U}_\mathcal{F}^{(3)} \otimes \bm{U}_\mathcal{F}^{(2)})] \}$. It can be verified that $\vectorize(\cm{Z}^{i}) = \bm{U}_G^{\prime}\vectorize(\cm{X}^{i})$, i.e., $\bm{U}_G$ is the fixed operation on the inputs.

\begin{assumption} \label{assum1}
	(i)  $\bar{\bm{x}} := (\vectorize(\cm{X}^{1})^{\prime},\vectorize(\cm{X}^{2})^{\prime},\ldots,\vectorize(\cm{X}^{n})^{\prime})^{\prime}$ is normally distributed with mean zero and variance $\bm{\Sigma}=\mathbb{E}(\bar{\bm{x}}\bar{\bm{x}}^{\prime}) $, where $c_x\bm{I}\leq\bm{\Sigma}\leq C_x\bm{I}$ for some universal constants $0<c_x<C_x$.  (ii)
	$\{\xi^i\}$ are independent $\sigma^2$-sub-Gaussian random variables with mean zero, and is independent of $\{\cm{X}^j, 1\leq j\leq i\}$ for all $1\leq i\leq n$.  (iii)
	$c_u\bm{I} \leq \bm{U}_G'\bm{U}_G\leq C_u\bm{I}$ for some universal constants  $0<c_u<C_u$.
\end{assumption}

\begin{theorem}[Sample complexity of CNN]
	\label{thm:errorbound-1}
	Under Assumption \ref{assum1}, if 
	$n\gtrsim d_{\mathcal{M}}^u$,  then 
	with probability at least $1-c\exp\{-(c_1n-c_2d_{\mathcal{M}}^u)\}-\exp\{-cd_{\mathcal{M}}^u\}$, it holds 
	$
	\cm{E}(\cm{\widehat{W}})\lesssim  \sqrt{{d_{\mathcal{M}}^u}/{n}} 
	$, where $c, c_1, c_2>0$ are universal constants and 
	\begin{equation}\label{samcom1}
		d_{\mathcal{M}}^u=
		K(P+L+1)\quad \text{if }K<\min(P,L) \quad\text{or}\quad
		d_{\mathcal{M}}^u=PL\quad \text{if }K\geq\min(P,L)
	\end{equation}
\end{theorem}
The above theorem states that the prediction error $\cm{E}(\cm{\widehat{W}})$ scales as $\sqrt{d_{\mathcal{M}}^u/n}$ with high probability, 
which implies that the number of samples required to achieve prediction error $\varepsilon$ is of order $d_{\mathcal{M}}^u/\varepsilon^2$; henceforth, we may regard $d_{\mathcal{M}}^u$ as the  sample complexity of the uncompressed CNN.
Note that $d_{\mathcal{M}}^u$ in \eqref{samcom1} is roughly equal to the naive parameter counting in \eqref{firstparamno} when $K$ is small. However, when $K\geq \min(P,L)$, $d_{\mathcal{M}}^u=PL$ can be much smaller than the naive parameter counting $K(P+L)$.

Next we consider the compressed CNN in Section 2.2. Training this model is equivalent to searching for the least-square estimator in \eqref{lse} with $\cm{A}_\mathrm{stack}$ subject to the structural constraint in the form of \eqref{Tucker}.
The following theorem provides sample complexities for compressed CNNs via the Tucker or CP decomposition, where the trained wieghts are denoted by $\cm{\widehat{W}}_{\mathrm{TU}}$ and $\cm{\widehat{W}}_{\mathrm{CP}}$, respectively.

\begin{theorem}[Sample complexity of compressed CNN]\label{thm:tucker}
	Under Assumption \ref{assum1}, if $n\gtrsim d_{\mathcal{M}}^c$, then
	with probability at least $1-c\exp\{-[c_1n-c_2d_{\mathcal{M}}^c]\}-\exp\{ -cd_{\mathcal{M}}^c\}$, it holds 
	$
	\cm{E}(\cm{\widehat{W}}_{\mathrm{TU}})\lesssim \sqrt{{d_{\mathcal{M}}^c}/{n}}$, where $c, c_1, c_2>0$ are universal  constants and 
	\begin{align}\label{compress-cnn}
		d_{\mathcal{M}}^c=
		R\prod_{i=1}^{N}R_i+\sum_{i=1}^{N}l_iR_i+RP.
	\end{align}
	Moreover, since CP is a special case of Tucker, if $n\gtrsim \widetilde{d}_{\mathcal{M}}^c$, then  	with probability at least $1-c\exp\{-[c_1n-c_2\widetilde{d}_{\mathcal{M}}^c]\}-\exp\{ -c\widetilde{d}_{\mathcal{M}}^c\}$, it holds $\cm{E}(\cm{\widehat{W}}_{\mathrm{CP}})\lesssim \sqrt{{\widetilde{d}_{\mathcal{M}}^c}/{n}}$, where $\widetilde{d}_{\mathcal{M}}^c$ is defined by setting  $R_i=R$ in $d_{\mathcal{M}}^c$ above, for all $1\leq i\leq N$.
\end{theorem}
Comparing  the sample complexity of the uncompressed CNN, $d_{\mathcal{M}}^u$, to that of the compressed CNN, $d_{\mathcal{M}}^c$,  the term $L=\prod_{i=1}^{N}l_i$ is shrunk to $\sum_{i=1}^{N}l_iR_i$. This verifies that when the kernel sizes for high-order convolution are large, the compressed CNN  indeed has a much smaller number of parameters than the uncompressed one.

However, more importantly, comparing the sample complexity of the compressed CNN,  $d_{\mathcal{M}}^c$, to the naive parameter counting $d_{\#}^c$ in \eqref{secondparamno}, an interesting discrepancy can be observed. Specifically, 
when $K>R$, i.e. the low rank constraint is imposed on the output channels, it is easy to see that $RK+KP>RP$ and hence, $d_{\#}^c > d_{\mathcal{M}}^c$. This implies that,  for a given prediction error $\varepsilon$, the number of  parameters in the compressed CNN is actually larger than the number of parameters that is necessary.

\section{Proposed approach to measuring model redundancy}
\subsection{A new measure for model redundancy}

Motivated by the discrepancy between $d_{\mathcal{M}}^c$ and $d_{\#}^c$, we establish theoretically the existence of the model redundancy in the compressed CNN in the following corollary.

\begin{corollary}[Redundancy in compressed CNN]\label{cor:CP&Tucker}
	When the activation function is linear, if $\cm{A}_\mathrm{stack}$ has a Tucker decomposition with the ranks $(R_1,\cdots,R_N, R)$, then
	the compressed CNN can be reparameterized into one with $R$ kernels, 
	each of which has
	Tucker ranks of $(R_1,\cdots,R_{N})$.
	Similarly,	if $\cm{A}_\mathrm{stack}$ has a CP decomposition with rank $R$, then the compressed CNN is equivalent to one with $R$ kernels, 
	each of which has a CP 
	rank  $R$.
\end{corollary}

To understand the above corollary, suppose that $F_2$ is a compressed CNN model via the Tucker decomposition with arbitrary $K>R$. By Theorem \ref{thm:tucker}, its sample complexity is $d_{\mathcal{M}}^c$; i.e.,  to achieve a given prediction error $\varepsilon$,  the number of samples required scales as $d_{\mathcal{M}}^c/\varepsilon^2$.  Holding $R$ fixed, by Corollary \ref{cor:CP&Tucker},  we can always find another model $F_1$ by setting $K=R$,  such that $F_1\subset F_2$, yet the corresponding sample complexity remains to be  $d_{\mathcal{M}}^c$. In other words, $F_1$ and $F_2$ can achieve the same prediction performance  with the same number of samples, despite that $F_1\subset F_2$.  This indicates that $F_2$ contains redundant parameters relative to $F_1$. Indeed, $F_1$, with the choice of $K=R$, is the most efficient compressed CNN model among the equivalent class of models with $K\geq R$. However, if we further take into account the nonlinearity in the model and data, the most efficient model may be the one with $K$ slightly larger than $R$.

The above finding inspires us to propose an easy-to-use measure of the model redundancy in compressed CNNs, calculated as the ratio  of $K/R$. This new measure provides a useful guidance for designing CNN design in practice.  Specifically, we give the following empirical recommendations:

(i) If a CNN model is built from scratch, one can first choose $R$, i.e. the channel sizes for the bottleneck layers.  Then, to minimize the possible model redundancy, it is recommended to choose $K$, i.e. the channel sizes for the output channels such that  the $K/R$ ratio is close to 1. 

(ii) If a CNN model with $K/R\gg 1$ is already employed, then it is recommended to reduce $K$ and try a smaller CNN model. If the prediction error for both models are roughly comparable over different datasets, then  the smaller model is preferred, as it enjoys higher parameter efficiency.

\subsection{Applications to mainstream compressed CNNs}
The idea of $K/R$ ratio can be applied to many block designs; see also Figure \ref{fig:bottleneck}. In what follows, we illustrate how to use it to evaluate model redundancy.

\paragraph{Standard bottleneck block}
The basic building block in ResNet \citep{he2016deep} can be exactly replicated by a Tucker decomposition on $\cm{A}_\mathrm{stack}\in\mathbb{R}^{l_1\times l_2\times C\times K}$, with ranks $(l_1, l_2, R, R)$, where $C$ is the number of input channels \citep{kossaifi2019factorized}. As we will show in the ablation studies in Section 4, when $K/R\gg 1$, this design may suffer from inefficient model redundancy.

\paragraph{Funnel block}
As a straightforward revision
of the standard bottleneck block, the funnel block maintains an output channel size of $K=R$ and thus is  an efficient block design. 

\paragraph{Depthwise separable block}
With the light-weight depthwise convolution, this block requires much fewer parameters than the standard bottleneck and is hence quite popular.  
It is equivalent to assuming a CP decomposition on $\cm{A}_\mathrm{stack}\in\mathbb{R}^{l_1\times l_2\times C\times K}$ with rank $R$ \citep{kossaifi2019factorized}. 
\citet{ma2018shufflenet} included it into the basic unit of ShuffleNetV2, sandwiched between channel splitting and concatenation. By removing the first $1\times 1$ pointwise convolution layer and setting $R = C$, this block then corresponds to the basic module in MobileNetV1 \citep{howard2017mobilenets}.
Our theoretical analysis suggests that model redundancy may exist in this block design when $K/R\gg 1$.

\paragraph{Inverted residual block}
\citet{sandler2018mobilenetv2} later proposed this design in MobileNetV2. It
ncludes expansive layers between the input and output layers, with the channel size of $x\cdot C (x\geq 1)$, where $x$ represents the expansion factor. 
As discussed in \citet{kossaifi2019factorized}, it  heuristically corresponds to a CP decomposition on $\cm{A}_{stack}$, with CP rank equals to $x\cdot C$. 
Since, the rank of output channel dimension can be at most $x\cdot C$, as long as $K \leq x\cdot C$, it is theoretically efficient and provides leeway in exploring thicker layers within blocks.

While our theoretical analysis is conducted under a simple framework, 
we will show through ablation studies in the next section that our finding indeed applies to a wide range of realistic scenarios.

\section{Ablation studies on the model redundancy for compressed block designs}

\begin{figure}[ht]
	\centering
	\includegraphics[width=0.9\linewidth]{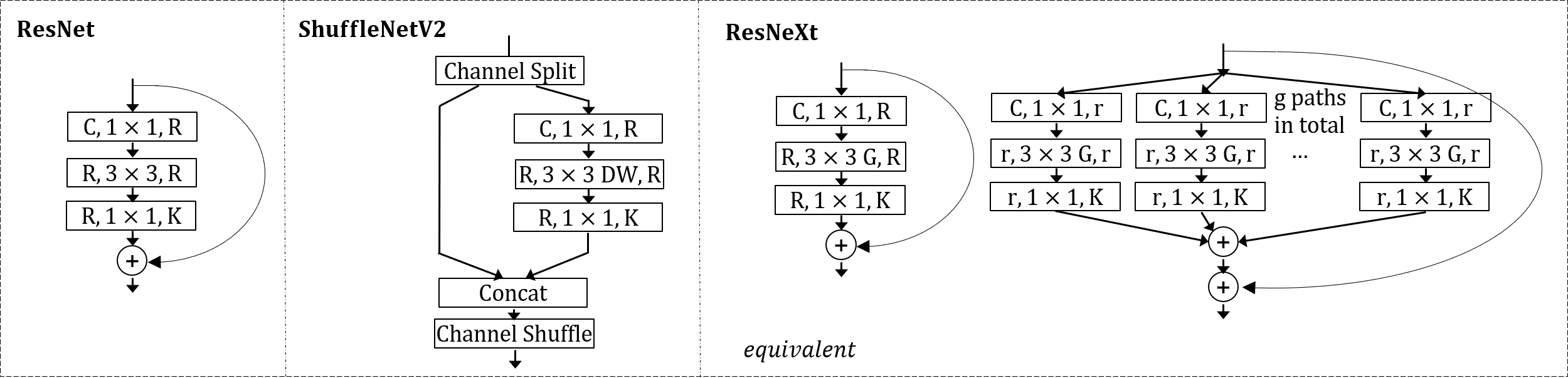}
	\caption{
		The bottleneck blocks in three popular networks. A layer is denoted as (\# input channels, filter size, \# ouput channels), where "G", "DW" denote group and depthwise convolution, respectively. Two equivalent forms of ResNeXt block structure is provided, where $r = R/g$ and $g$ is the number of groups.
	}
	\label{fig:real}
\end{figure}

In this section, we study how the $K/R$ ratio may influence the accuracy vs. parameter efficiency trade-off in three popular networks, namely  ResNet \citep{he2016deep}, ResNeXt \citep{xie2017aggregated} and ShuffleNetV2 \citep{ma2018shufflenet}. In Figure \ref{fig:real}, we present the basic bottleneck block structures in these networks. 
The block structures in ResNet and ShuffleNetV2 correspond to the standard bottleneck block and the depthwise separable block, both of which are discussed in detail in Section 4. We will show 
below that the block structure in ResNeXt is a variation of the standard bottleneck block and its $K/R$ ratio can be similarly defined. 

\paragraph{Bottleneck with group convolution}
As shown in the right panel of Figure \ref{fig:real}, the bottleneck block of ResNeXt performs group convolution in the middle layer. Essentially, it divides its input channels into $g$ groups, each of size $r = R/g$, and performs regular convolution in each group. The outputs are then concatenated into a single output of this middle layer. Like the standard bottleneck block in ResNet, 
this block is equivalent to assuming a Tucker decomposition \eqref{Tucker} on $\cm{A}_\mathrm{stack}\in\mathbb{R}^{l_1\times l_2\times C\times K}$, with ranks of $(l_1, l_2, R, R)$. However, the core tensor $\cm{G}$ has at most $g$ non-overlapping sub-blocks, each of size $\mathbb{R}^{l_1\times l_2\times r\times r}$, which have non-zero entries. Hence, the theoretical implications can be applied.

\begin{figure*}[t]
	\centering
	\includegraphics[width=1.\linewidth]{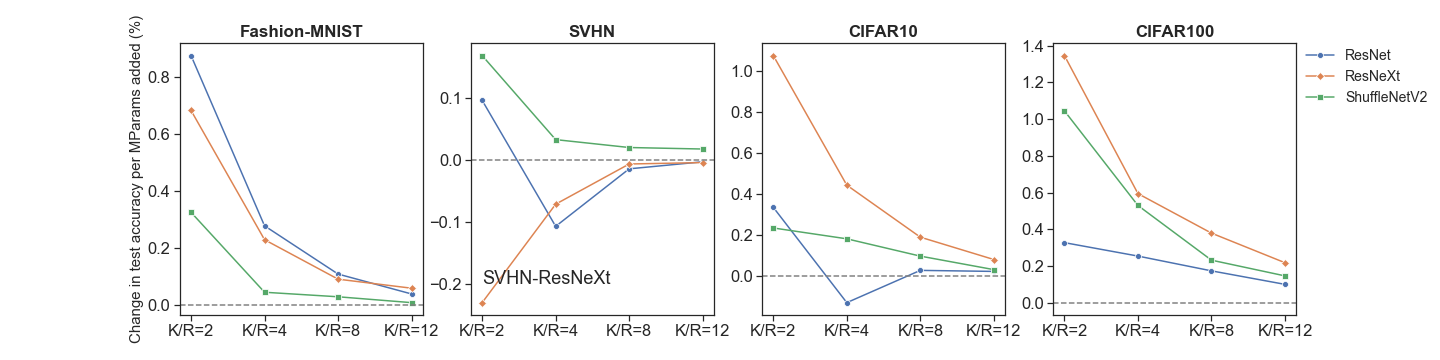}
	\caption{The change in test accuracy (\%) per millions of parameters added for different $K/R$ ratios, comparing to the baseline ($K/R=1$).}
	\label{fig:accuracy}
\end{figure*}

\paragraph{Data}
We analyze four image recognition datasets, Fashion-MNIST \citep{xiao2017fashion}, CIFAR-10 \& CIFAR-100 \citep{krizhevsky2009learning} and Street View House Numbers (SVHN) \citep{netzer2011reading}.
For CIFAR-10 \& CIFAR-100 and SVHN datasets, we adopt the data pre-processing and augmentation techniques in \citet{he2016deep}. For Fashion-MNIST dataset, we simply include random horizontal flipping for data augmentation.

\paragraph{Network architecture}
The networks we adopt are based on the standard ResNet-50 in \citet{he2016deep}, ResNeXt-50 in \citet{xie2017aggregated} and ShuffleNetV2-50 in \citet{ma2018shufflenet}. We uniformly use $3\times 3$ kernels in the first convolution layer and delete the "conv5" layers to avoid overfitting. Hence, each network has $41$ layers, with $13$ bottleneck blocks with structures presented in Figure \ref{fig:real}. The bottleneck blocks can be sequentially divided into three groups of sizes 3, 4 and 6, according to the different values of $R$. For both ResNet and ResNeXt, we set the value of $R$ to be $\{64, 128, 256\}$ for the three groups, respectively, and downsampling is performed by 1st, 4th and 8th bottleneck blocks in the same way as in \citet{he2016deep} and \citet{xie2017aggregated}. The group convolution in ResNeXt is performed with $32$ groups. For the ShuffleNetV2, we let the number of kernels for the first convolution layer  be $24$, and then set the value of $R$ to  $\{58, 116, 232\}$ for the three groups, respectively. For our ablation studies, the $K/R$ ratio takes values in $\{1, 2, 4, 8, 12\}$. Then the size of $K$ in each bottleneck block is determined by $(K/R)\cdot R$.

\paragraph{Implementation details}
All experiments are conducted in PyTorch  on Tesla V100-DGXS.
We follow the common practice for ResNet, ResNeXt and ShuffleNetV2 to adopt batch normalization \citep{ioffe2015batch} after convolution and before the ReLU activation. 
The weights are initialized as in \citet{he2015delving}. We use stochastic gradient descent with weight decay $10^{-4}$, momentum 0.9 and mini-batch size 128. The learning rate starts from 0.1 and is divided by 10 for every 100 epochs.
We stop training after 350 epochs, since the training accuracy hardly changes.
We repeat the experiment under each model and dataset by setting random seeds 1--3 and report the worst case scenario as our Top-1 test accuracy.

\paragraph{Results}
The results are presented in Figure \ref{fig:accuracy}, where we use the following criterion to evaluate the efficiency of the added parameters under different $K/R$ ratios against the baseline ($K/R=1$):
\begin{align*}
	\text{Change in test accuracy per MParams added} = \frac{\text{Change in test accuracy}}{\text{\# Parameters added in millions}}.
\end{align*}
For each given dataset and model, this criterion basically shows that, as the $K/R$ ratio changes from 2 to 12, for every 1 million parameters we add onto the baseline model, how much change in the test accuracy can be expected.

The overall observation is that the change in test accuracy per MParam added converges to zero as the $K/R$ ratio grows, with the ``elbow" cutoff point appearing at $K/R=4$ or sometimes 2. When the lines are in the negative region, it indicates that the test accuracy is worse than the baseline model. For the case of ``SVHN-ResNeXt", we can see that its best test accuracy is, in fact, achieved at the baseline with $K/R=1$.

The detailed results of the test accuracy and the number of parameters and FLOPs under each setting are deferred to the supplementary material.

\paragraph{Practical implications}
Our theoretical results and ablation studies imply jointly that, the $K/R$ ratio is an important measure for model redundancy in CNN block designs. In practice, it is recommended to adopt $K/R$ close to 1 to achieve the optimal accuracy vs. parameter efficiency trade-off. In most cases, the tuning for $K/R$ only needs to be performed over the interval 1--4.

\section{Conclusion and discussion}

In this paper, a general statistical framework is introduced to answer the question on the existence of remaining model redundancy in a compressed CNN model. Then a new measure, the $K/R$ ratio, is proposed to quantify the model redundancy and provide empirical guidance on CNN designs.  Numerical studies further suggest that the optimal region for the ratio lies between the interval 1--4.

It is worthwhile to extend the proposed methodology in the following three directions. First, we can consider CNNs with more layers. For instance, for a 5-layer CNN with ``convolution $\rightarrow$ pooling $\rightarrow$ convolution $\rightarrow$ pooling $\rightarrow$ fully connected" structure, it can be readily shown that the corresponding model has the linear form of $y^i = \langle\cm{X}^i, \cm{W}_D\rangle + \xi^i$, with
\[
\cm{W}_D = (\sum_{k_1=1}^{K_1}\sum_{k_2=1}^{K_2}\cm{B}_{k_1,k_2}\otimes\cm{\widetilde{A}}_{k_2}\otimes\cm{A}_{k_1})\times_1\bm{U}_\mathcal{DF}^{(1)}\times_2\cdots\times_N\bm{U}_\mathcal{DF}^{(N)},
\] 
where $\{\cm{B}_{k_1,k_2}\}$ are the fully-connected weights, and $\{\cm{A}_{k_1}\}$ and $\{\cm{\widetilde{A}}_{k_2}\}$ are the kernel tensors for the first and second convolution layers, respectively;  the detailed derivation of the above is provided in the supplementary file. Then the sample complexity analysis can be conducted similarly to the deeper CNN models.	
Secondly, in this paper, in order to leverage existing technical tools from high-dimensional statistics, our sample complexity analysis is conducted under the assumption of linear activations. The extension to nonlinear activations is highly nontrivial yet of great interest to theorists. 

Lastly, it is interesting to generalize the proposed  framework to other network structures, such as  variants of the compressed CNN block design, e.g., one that incorporates the compression of  fully-connected layers via the tensor decomposition \citep{kossaifi2017tensor, kossaifi2020tensor}, as well as  other more complex network architectures such as  RNN and GNN. For the latter,  development of corresponding statistical formulations and analysis of the discrepancy between sample complexities and naive parameter countings similarly to the present paper would be necessary, which  we leave  for future research.

\medskip
\small
\bibliography{CNN}

\newpage
\renewcommand{\thesection}{Appendix A}
\renewcommand{\thesubsection}{A\arabic{subsection}}
\renewcommand{\theequation}{A\arabic{equation}}
\renewcommand{\thelemma}{A\arabic{lemma}}
\setcounter{lemma}{0}
\setcounter{equation}{0}

\section{CNN formulation}

\subsection{One-layer CNN formulation}

For a tensor input $\cm{X}\in\mathbb{R}^{d_1\times\cdots\times d_N}$, it first convolutes with an $N$-dimensional kernel tensor $\cm{A}\in\mathbb{R}^{l_1\times\cdots\times l_N}$ with stride sizes equal to $s_c$, and then performs an average pooling with pooling sizes equal to $(q_1,\cdots,q_N)$. It ends with a fully-connected layer with the weight tensor $\cm{B}$ and produces a scalar output. We assume that $m_j=(d_j-l_j)/s_c+1$ are integers for $1\leq j\leq N$, otherwise zero-padding will be needed. For ease of notation, we take the pooling sizes $\{q_j\}_{j=1}^N$ to satisfy the relationship $m_j=p_jq_j$. 

To duplicate the operation of the convolution layer using a simple mathematical expression, we first need to define a set of matrices $\{\bm{U}_{i_j}^{(j)}\in\mathbb{R}^{d_j\times l_j}\}$, where
\begin{equation}
	\bm{U}_{i_j}^{(j)} = [\underbrace{\bm{0}}_{(i_j-1)s_c} \underbrace{\bm{I}}_{l_j} \underbrace{\bm{0}}_{d_j-(i_j-1)s_c-l_j}]^\prime \in\mathbb{R}^{d_j\times l_j}.
\end{equation}
for $1\leq i_j\leq m_j$, $1\leq j\leq N$. $\bm{U}_{i_j}^{(j)}$ acts as a positioning factor to transform the kernel tensor $\cm{A}$ into a tensor of same size as the input $\cm{X}$, with the rest of the entries equal to zero.

We are ready to construct our main formulation for a 3-layer tensor CNN. To begin with, we first illustrate the process using a vector input $\bm{x}\in\mathbb{R}^{d}$, with a kernel vector $\bm{a}\in\mathbb{R}^{l}$; see Figure \ref{fig:cnn-vec}. Using the $i\ts{th}$ positioning matrix $\bm{U}_{i}^{(1)}$, we can propagate the small kernel vector $\bm{a}$, into a vector $\bm{U}_{i}^{(1)}\bm{a}$ of size $\mathbb{R}^d$, by filling the rest of the entries with zeros. 
The intermediate output vector has entries given by $\bm{x}_c(i) = \langle\bm{x},\bm{U}^{(1)}_i\bm{a}\rangle$, for $1\leq i\leq m$.
There are a total of $p$ pooling blocks in $\bm{x}_c$, with the $k$th block formed by $\{\bm{x}_c(i)\}_{i\in S_{k}}$, where the index set
\[
S_{k} = \{(k-1)q+1\leq i\leq kq\}, 1\leq k\leq p.
\]
By taking the average per block, the resulting vector $\bm{x}_{cp}$ is of size $\mathbb{R}^{p}$, with $i$th entry equal to $\bm{x}_{cp}(k) = q^{-1}\sum_{i\in S_k}\bm{x}_c(i) = \langle\bm{x},q^{-1}\sum_{i\in S_k}\bm{U}_{i}^{(1)}\bm{a}\rangle$. 
The fully-connected layer performs a weighted summation over the $p$ vectors, with weights given by entries of the vector $\bm{b}\in\mathbb{R}^p$.
This gives us the predicted output $\widehat{y} = \langle\bm{b},\bm{x}_{cp}\rangle = \langle\bm{x},\bm{w}_X\rangle$, where
\begin{equation*}
	\begin{split}
		\bm{w}_X &= \sum_{k=1}^{p}b_k\frac{1}{q}\sum_{i\in S_k}\bm{U}^{(1)}_i\bm{a}= \bm{U}^{(1)}_\mathcal{F}(\bm{b}\otimes\bm{a}) = (\bm{b}\otimes\bm{a}) \times_1 \bm{U}^{(1)}_\mathcal{F},
	\end{split}
\end{equation*}
and $\bm{U}^{(1)}_\mathcal{F} = q^{-1}(\sum_{i\in S_1}\bm{U}_{i}^{(1)},\cdots,\sum_{i\in S_p}\bm{U}_{i}^{(1)})$, and "$\times_1$" represents the mode-1 product.

For matrix input $\bm{X}$ with matrix kernel $\bm{A}$, however, we need 2 sets of positioning matrices, $\{\bm{U}_{i_1}^{(1)}\}_{i_1=1}^{m_1}$ and $\{\bm{U}_{i_2}^{(2)}\}_{i_2=1}^{m_2}$, one for the height dimension and the other for width. Then, the intermediate output from convolution has entries given by $\bm{X}_c(i_1,i_2) = \langle\bm{X},\bm{U}^{(1)}_{i_1}\bm{A}\bm{U}^{(2)\prime}_{i_2}\rangle$, for $1\leq i_1\leq p_1$, $1\leq i_2\leq p_2$.
For the average pooling, we form $p_1p_2$ consecutive matrices from $\bm{X}_c$, each of size $\mathbb{R}^{q_1\times q_2}$ and take the average.
This results in $\bm{X}_{cp}\in\mathbb{R}^{p_1\times p_2}$, with 
\[
\bm{X}_{cp}(i_1,i_2) = \langle\bm{X},(q_{1}^{-1}\sum_{i_1\in S_{k_1}}\bm{U}^{(1)}_{i_1})\bm{A}(q_{2}^{-1}\sum_{i_2\in S_{k_2}}\bm{U}^{(2)}_{i_2})^\prime\rangle,
\] 
where, for $1\leq k_j\leq p_j$,
\[
S_{k_j} = \{(k_j-1)q_j+1\leq i\leq k_jq_j: 1\leq j\leq 2\}.
\]
The output $\bm{X}_{cp}$ goes through fully-connected layer with weight matrix $\bm{B}$ and gives the predicted output $\widehat{y} = \langle\bm{X},\bm{W}_X\rangle$, where
\begin{equation*}
	\begin{split}
		\bm{W}_X &= \sum_{k_2=1}^{p_2}\sum_{k_1=1}^{p_1}b_{k_1}b_{k_2}\frac{1}{q_1q_2}(\sum_{i_1\in S_{k_1}}\bm{U}^{(1)}_{i_1})\bm{A}(\sum_{i_2\in S_{k_2}}\bm{U}^{(2)\prime}_{i_2})\\
		& = (\bm{B}\otimes\bm{A}) \times_1 \bm{U}^{(1)}_\mathcal{F}\times_2 \bm{U}^{(2)}_\mathcal{F},	
	\end{split}
\end{equation*}
with $\bm{U}^{(j)}_\mathcal{F} = q_j^{-1}(\sum_{i\in S_1}\bm{U}_{i}^{(j)},\cdots,\sum_{i\in S_{p_j}}\bm{U}_{i}^{(j)})$ for $j = 1$ or 2. And "$\times_1$", "$\times_2$" represent the mode-1 and mode-2 product.

\begin{figure*}[t]
	\centering
	\includegraphics[width=\linewidth]{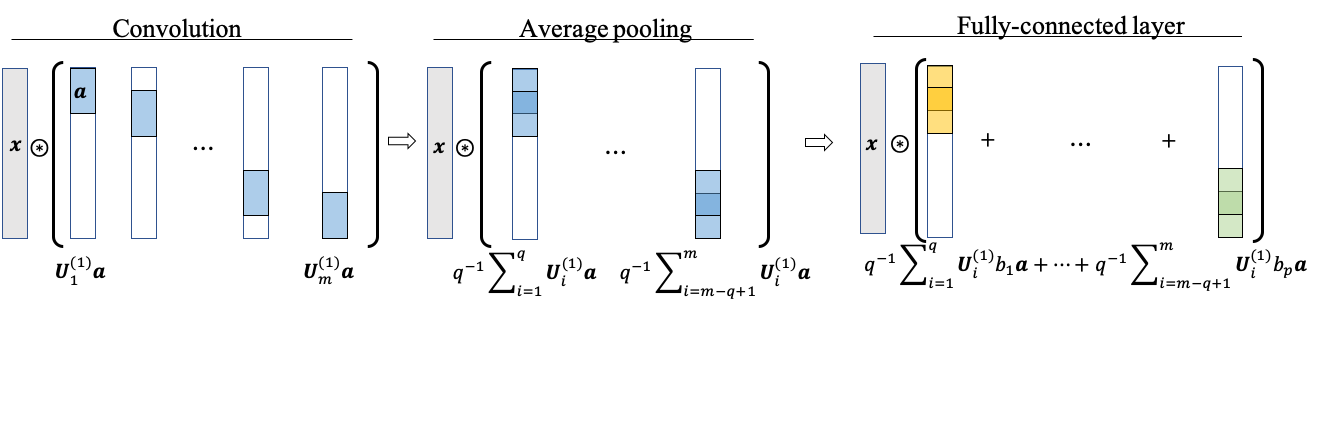}
	\caption{Formulating the process with an input vector $\bm{x}$. Note that we combine the convolution, pooling and fully connected layers into the composite weight vector $\bm{w}_X$, where $\bm{w}_X= q^{-1}\sum_{i\in S_1}\bm{U}_{i}^{(1)}\bm{a}+\cdots+q^{-1}\sum_{i\in S_p}\bm{U}_{i}^{(1)}\bm{a}$. The $\circledast$ represents the convolution operation.}
	\label{fig:cnn-vec}
\end{figure*}

From here, together with the case of high-order tensor input discussed in Session 2, we can then derive the form of predicted outcome as $\widehat{y} = \langle\cm{X},\cm{W}_X^\mathrm{single}\rangle$, where
\begin{equation*}
	\cm{W}_X^\mathrm{single} = (\cm{B}\otimes\cm{A}) \times_1 \bm{U}^{(1)}_\mathcal{F}\times_2 \bm{U}^{(2)}_\mathcal{F}\times\cdots\times \bm{U}^{(N)}_\mathcal{F},
\end{equation*}
with $\bm{U}^{(j)}_\mathcal{F} = q_j^{-1}(\sum_{i\in S_{1}}\bm{U}_{i}^{(j)},\cdots,\sum_{i\in S_{p_j}}\bm{U}_{i}^{(j)})$, and "$\times_j$" represents the mode-$j$ product for $1\leq j\leq N$.

With $K$ kernels, we denote the set of kernels and the corresponding fully-connected weight
tensors as $\{\cm{A}_k, \cm{B}_k\}_{k=1}^K$. Since the convolution and pooling operations are identical across kernels, we can use a summation over kernels to derive the weight tensor for multiple kernels, which is
\begin{equation*}
	\cm{W}_X = (\sum_{k=1}^{K}\cm{B}_k\otimes\cm{A}_k) \times_1 \bm{U}^{(1)}_\mathcal{F}\times_2 \bm{U}^{(2)}_\mathcal{F}\times\cdots\times \bm{U}^{(N)}_\mathcal{F},
\end{equation*}
and we arrive at the formulation for 3-layer tensor CNN.

\subsection{Five-layer CNN formulation}\label{app:cnn}

Now, we can consider a 5-layer CNN with "convolution$\rightarrow$pooling$\rightarrow$convolution$\rightarrow$pooling$\rightarrow$fully connected" layers and a 3D tensor input $\cm{X}\in\mathbb{R}^{d_1\times d_2\times d_3}$.
Here, we denote the intermediate output from the first convolution by ${\cm{X}}_c\in\mathbb{R}^{{m}_1\times{m}_2\times{m}_3}$, from the second convolution by $\cm{\widetilde{X}}_c\in\mathbb{R}^{\tilde{m}_1\times\tilde{m}_2\times\tilde{m}_3}$, and the output from the first pooling by ${\cm{X}}_{cp}\in\mathbb{R}^{{p}_1\times{p}_2\times{p}_3}$
and from the second pooling by $\cm{\widetilde{X}}_{cp}\in\mathbb{R}^{\tilde{p}_1\times\tilde{p}_2\times\tilde{p}_3}$.
We can first see that the predicted output from the 5-layer CNN is the same as directly feeding $\cm{X}_{cp}$ to the second convolution layer, followed by an average pooling layer and a fully-connected layer. 

Denote the first convolution kernel tensor by $\cm{A}\in\mathbb{R}^{l_1\times l_2\times l_3}$, 
the second convolution kernel tensor by $\cm{A}\in\mathbb{R}^{\tilde{l}_1\times \tilde{l}_2\times \tilde{l}_3}$,
and the fully-connected weight tensor by
$\cm{B}\in\mathbb{R}^{\tilde{p}_1\times\tilde{p}_2\times\tilde{p}_3}$. Define the set of matrices $\{\bm{U}_{i_j}^{(j)}\in\mathbb{R}^{d_j\times l_j}\}$, for $1\leq i_j\leq m_j, 1\leq j\leq N$ as (2) in the paper. And similarly, define $\{\widetilde{\bm{U}}_{i_j}^{(j)}\in\mathbb{R}^{p_j\times \tilde{l}_j}\}$, with
\begin{equation}
	\widetilde{\bm{U}}_{i_j}^{(j)} = [\underbrace{\bm{0}}_{(i_j-1)\tilde{s}_c} \underbrace{\bm{I}}_{\tilde{l}_j} \underbrace{\bm{0}}_{p_j-(i_j-1)\tilde{s}_c-\tilde{l}_j}]^\prime,
\end{equation}
where $\tilde{s}_c$ is the stride size for the second convolution. Now, let $\widetilde{\bm{U}}^{(j)}_{\mathcal{F}} = (\widetilde{\bm{U}}^{(j)}_{\mathcal{F},1},\cdots,\widetilde{\bm{U}}^{(j)}_{\mathcal{F},p_j}) = (\tilde{q}_j^{-1}\sum_{k=1}^{\tilde{q}_j}\widetilde{\bm{U}}_{k}^{(j)},\cdots,\tilde{q}_j^{-1}\sum_{k=\tilde{m}_j-\tilde{q}_j+1}^{\tilde{m}_j}\widetilde{\bm{U}}_{k}^{(j)})$.
We further stack the matrices $\{\bm{U}_{\mathcal{F},i_j}^{(j)}\in\mathbb{R}^{d_j\times l_j}\}_{1\leq i_j\leq p_j}$ and $\{\widetilde{\bm{U}}_{\mathcal{F},i_j}^{(j)}\in\mathbb{R}^{p_j\times\tilde{l}_j}\}_{1\leq i_j\leq \tilde{p}_j}$ into 3D tensors $\cm{U}^{(j)}\in\mathbb{R}^{d_j\times l_j\times p_j}$ and $\cm{\widetilde{U}}^{(j)}\in\mathbb{R}^{p_j\times \tilde{l}_j\times\tilde{p}_j}$.

We have that the predicted output
\begin{equation*}
	\begin{split}
		\widehat{y} &= \langle\cm{B}\otimes\cm{\widetilde{A}},
		[\![\cm{X}_{cp};\widetilde{\bm{U}}_\mathcal{F}^{(1)\prime},\widetilde{\bm{U}}_\mathcal{F}^{(2)\prime},\widetilde{\bm{U}}_\mathcal{F}^{(3)\prime}]\!]\rangle\\
		&=\Big\langle\cm{B}\otimes\cm{\widetilde{A}}, \sum_{j=1}^{3}\sum_{i_j=1}^{p_j}\langle\cm{A},\cm{T}_1(i_1,i_2,i_3)\rangle\widetilde{\bm{u}}_{i_1}\circ\widetilde{\bm{u}}_{i_2}\circ\widetilde{\bm{u}}_{i_3}\Big\rangle\\
		&=\Big\langle\cm{B}\otimes\cm{\widetilde{A}}\otimes\cm{A}, \sum_{j=1}^{3}\sum_{i_j=1}^{p_j}\cm{T}_2(i_1,i_2,i_3)\Big\rangle\\
		& =\Big\langle\cm{B}\otimes\cm{\widetilde{A}}\otimes\cm{A}, \cm{T}_3\Big\rangle\\
		&=\Big\langle\cm{X},\cm{T}_4\Big\rangle\\
		&=\Big\langle\cm{X},
		(\cm{B}\otimes\cm{\widetilde{A}}\otimes\cm{A})
		\times_1\bm{U}_\mathcal{DF}^{(1)}
		\times_2\bm{U}_\mathcal{DF}^{(2)}
		\times_3\bm{U}_\mathcal{DF}^{(3)}
		\Big\rangle,
	\end{split}
\end{equation*}
where
\begin{align*}
	&\cm{T}_1(i_1,i_2,i_3) = [\![\cm{X};\bm{U}_{\mathcal{F},i_1}^{(1)\prime},\bm{U}_{\mathcal{F},i_2}^{(2)\prime},\bm{U}_{\mathcal{F},i_3}^{(3)\prime}]\!],\\
	&\cm{T}_2(i_1,i_2,i_3) = 
	[\![\cm{X};
	\widetilde{\bm{u}}_{i_1}\otimes\bm{U}_{\mathcal{F},i_1}^{(1)\prime},
	\widetilde{\bm{u}}_{i_2}\otimes\bm{U}_{\mathcal{F},i_2}^{(2)\prime},
	\widetilde{\bm{u}}_{i_3}\otimes\bm{U}_{\mathcal{F},i_3}^{(3)\prime}]\!],\\
	&\cm{T}_3 = [\![\cm{X},
	(\cm{U}^{(1)}\times_3\cm{\widetilde{U}}^{(1)})_{(1)}^\prime,
	(\cm{U}^{(2)}\times_3\cm{\widetilde{U}}^{(2)})_{(1)}^\prime,
	(\cm{U}^{(3)}\times_3\cm{\widetilde{U}}^{(3)})_{(1)}^\prime]\!],\\
	&\cm{T}_4 =
	[\![\cm{B}\otimes\cm{\widetilde{A}}\otimes\cm{A},
	(\cm{U}^{(1)}\times_3\cm{\widetilde{U}}^{(1)})_{(1)},
	(\cm{U}^{(2)}\times_3\cm{\widetilde{U}}^{(2)})_{(1)},
	(\cm{U}^{(3)}\times_3\cm{\widetilde{U}}^{(3)})_{(1)}]\!],
\end{align*}
and $\widetilde{\bm{u}}_{i_j} = \vectorize(\cm{\widetilde{U}}^{(1)}(i_j,:,:))\in\mathbb{R}^{\tilde{l}_j\tilde{p}_j}$ and $\bm{U}_\mathcal{DF}^{(j)} = (\cm{U}^{(j)}\times_3\cm{\widetilde{U}}^{(j)})_{(1)}$, for $1\leq i_j\leq p_j$ and $1\leq j\leq 3$.

\newpage

\renewcommand{\thesection}{Appendix B}
\renewcommand{\thelemma}{B.\arabic{lemma}}
\renewcommand{\thesubsection}{B.\arabic{subsection}}
\renewcommand{\theequation}{B.\arabic{equation}}
\setcounter{lemma}{0}
\section{Additional Experiments}

\subsection{Detailed results of the ablation studies}
In this section, we provide the detailed results for our ablation studies on ResNet, ResNeXt, ShuffleNetV2 on four datasets, namely Fashion-MNIST, SVHN, CIFAR10 and CIFAR100. Note that all datasets and models we adopted are under the MIT license.
\begin{table}[ht]
	\centering
	\caption{Results on Fashion-MNIST.}
	\label{tab:Fashion}
	\resizebox{0.55\textwidth}{!}{
		\begin{tabular}{ccccc}
			\toprule
			Model &$K/R$ &top-1 accuracy(\%) &GMacs &\#MParams\\
			\midrule
			\multirow{4}{*}{\makecell{ResNet}} 
			&1&94.26&0.16&5.20\\
			&2&94.27&0.19&6.24\\
			&4&94.29&0.26&8.55\\
			&8&94.34&0.43&14.15\\
			&12&93.99&0.63&21.06\\
			\cmidrule[0.5pt](lr{0.125em}){1-5}
			\multirow{4}{*}{\makecell{ResNeXt}}
			&1&93.97&0.04&1.09\\
			&2&94.07&0.07&2.13\\
			&4&94.13&0.14&4.44\\
			&8&94.18&0.31&10.04\\
			&12&94.31&0.52&16.95\\
			\cmidrule[0.5pt](lr{0.125em}){1-5}
			\multirow{4}{*}{\makecell{ShuffleNet}} 
			&1&93.36&0.04&0.78\\
			&2&93.65&0.08&1.67\\
			&4&93.50&0.16&3.84\\
			&8&93.63&0.38&9.81\\
			&12&93.52&0.65&17.93\\
			\bottomrule
		\end{tabular}
	}
\end{table}
\begin{table}[ht]
	\centering
	\caption{Results on SVHN.}
	\label{tab:SVHN}
	\resizebox{0.55\textwidth}{!}{
		\begin{tabular}{ccccc}
			\toprule
			Model &$K/R$ &top-1 accuracy(\%) &GMacs &\#MParams\\
			\midrule
			\multirow{4}{*}{\makecell{ResNet}} & 1 &96.45&0.16&5.20\\
			&2&96.55&0.19&6.24\\
			&4&96.09&0.26&8.55\\
			&8&96.32&0.43&14.15\\
			&12&96.40&0.63&21.06\\
			\cmidrule[0.5pt](lr{0.125em}){1-5}
			\multirow{4}{*}{\makecell{ResNeXt}} & 1 &96.44&0.04&1.09\\
			&2&96.20&0.07&2.13\\
			&4&96.20&0.14&4.44\\
			&8&96.38&0.31&10.04\\
			&12&96.37&0.52&16.95\\
			\cmidrule[0.5pt](lr{0.125em}){1-5}
			\multirow{4}{*}{\makecell{ShuffleNet}} & 1 &95.93&0.04&0.78\\
			&2&96.08&0.08&1.67\\
			&4&96.03&0.16&3.84\\
			&8&96.11&0.38&9.81\\
			&12&96.23&0.65&17.93\\
			\bottomrule
		\end{tabular}
	}
\end{table}
\begin{table}[ht]
	\centering
	\caption{Results on CIFAR10.}
	\label{tab:CIFAR10}
	\resizebox{0.55\textwidth}{!}{
		\begin{tabular}{ccccc}
			\toprule
			Model &$K/R$ &top-1 accuracy(\%) &GMacs &\#MParams\\
			\midrule
			\multirow{4}{*}{\makecell{ResNet}} &1&93.53&0.16&5.20\\
			&2&93.88&0.19&6.24\\
			&4&93.10&0.26&8.55\\
			&8&93.79&0.43&14.15\\
			&12&93.91&0.63&21.06\\
			\cmidrule[0.5pt](lr{0.125em}){1-5}
			\multirow{4}{*}{\makecell{ResNeXt}} & 1&91.89&0.04&1.09\\
			&2&93.01&0.07&2.13\\
			&4&93.38&0.14&4.44\\
			&8&93.60&0.31&10.04\\
			&12&93.19&0.52&16.95\\
			\cmidrule[0.5pt](lr{0.125em}){1-5}
			\multirow{4}{*}{\makecell{ShuffleNet}} &1&90.77&0.04&0.78\\
			&2&90.98&0.08&1.67\\
			&4&91.33&0.16&3.84\\
			&8&91.66&0.38&9.81\\
			&12&91.33&0.65&17.93\\
			\bottomrule
		\end{tabular}
	}
\end{table}
\begin{table}[ht]
	\centering
	\caption{Results on CIFAR100.}
	\label{tab:CIFAR100}
	\resizebox{0.55\textwidth}{!}{
		\begin{tabular}{ccccc}
			\toprule
			Model &$K/R$ &top-1 accuracy(\%) &GMacs &\#MParams\\
			\midrule
			\multirow{4}{*}{\makecell{ResNet}}&1 &74.46&0.16&5.20\\
			&2&74.80&0.19&6.24\\
			&4&75.31&0.26&8.55\\
			&8&76.01&0.43&14.15\\
			&12&76.04&0.63&21.06\\
			\cmidrule[0.5pt](lr{0.125em}){1-5}
			\multirow{4}{*}{\makecell{ResNeXt}} & 1&72.28&0.04&1.09\\
			&2&73.68&0.07&2.13\\
			&4&74.27&0.14&4.44\\
			&8&75.67&0.31&10.04\\
			&12&75.73&0.52&16.95\\
			\cmidrule[0.5pt](lr{0.125em}){1-5}
			\multirow{4}{*}{\makecell{ShuffleNet}} &1 &68.78&0.04&0.78\\
			&2&69.71&0.08&1.67\\
			&4&70.40&0.16&3.84\\
			&8&70.87&0.38&9.81\\
			&12&71.28&0.65&17.93\\
			\bottomrule
		\end{tabular}
	}
\end{table}

\newpage
\subsection{Numerical analysis for theoretical results}
\begin{table}[t]
	\caption{Different settings for verifying Theorem 1 (left), and Theorem 2 (right). }
	\label{tab:simulation}
	\centering
	\resizebox{1.0\textwidth}{!}{
		\begin{tabular}{lllll|lllll}
			\toprule
			&Input sizes&Kernel sizes&Pooling sizes&\# Kernels&Input sizes&Kernel sizes&Pooling sizes&Tucker ranks&\# Kernels\\
			\midrule
			Setting 1 (S1)& (7, 5, 7) & (2, 2, 2)& (3, 2, 3)& 1& (10, 10, 8, 3) & (5, 5, 3)& (3, 3, 3)& (2, 2, 2, 1)&2\\
			Setting 2 (S2)& (7, 5, 7) & (2, 2, 2)& (3, 2, 3)& 3& (10, 10, 8, 3) & (5, 5, 3)& (3, 3, 3)& (2, 2, 2, 1)&3\\
			Setting 3 (S3)& (8, 8, 3)& (3, 3, 3)& (3, 3, 1)& 1& (12, 12, 6, 3)& (7, 7, 3)& (3, 3, 2)& (2, 3, 2, 1)&2\\
			Setting 4 (S4)& (8, 8, 3)& (3, 3, 3)& (3, 3, 1)& 3& (12, 12, 6, 3)& (7, 7, 3)& (3, 3, 2)& (2, 3, 2, 1)&3\\
			\bottomrule
	\end{tabular}}
\end{table}

We choose four settings to verify the sample complexity in Theorem 1; see Table \ref{tab:simulation}. 
The parameter tensors $\{\cm{A}_k,\cm{B}_k\}_{k=1}^K$ are generated to have standard normal entries. We consider two types of inputs:
the independent inputs are generated with entries being standard normal random variables, and the time-dependent inputs $\{\bm{x}^i\}$ are generated from a stationary VAR(1) process. The random additive noises are generated from the standard normal distribution. The number of training samples $n$ varies such that $\sqrt{d_{\mathcal{M}}/{n}}$ is equally spaced in the interval $[0.15, 0.60]$, with all other parameters fixed. For each $n$, we generate 200 training sets to calculate the averaged estimation error in Frobenius norm. The result is presented in the Left panel of Figure \ref{fig:theorem1-2}. It can be seen that the estimation error increases linearly with the square root of $d_{\mathcal{M}}/n$, which is consistent with our finding in Theorem 1. 

For Theorem 2, we also adopt four different settings; see Table \ref{tab:simulation}. Here, we consider 4D input tensors. The stacked kernel $\cm{A}_\mathrm{stack}$ is generated by (6), where the core tensor has standard normal entries and the factor matrices are generated to have orthonormal columns. The number of training samples $n$ is similarly chosen such that $\sqrt{d_{\mathcal{M}}^\mathrm{TU}/{n}}$ is equally spaced. The result is presented in the Left panel of Figure \ref{fig:theorem1-2}. A linear trend between the estimation error and the square root of ${d_{\mathcal{M}}^\mathrm{TU}/{n}}$ can be observed, which verifies Theorem 2.

For details of implementation, we employed the gradient descent method for the optimization with a learning rate of 0.01 and a momentum of 0.9. The procedure is deemed to have reached convergence if the target function drops by less than $10^{-8}$. 

\begin{figure}
	\centering
	\includegraphics[width=1.\linewidth]{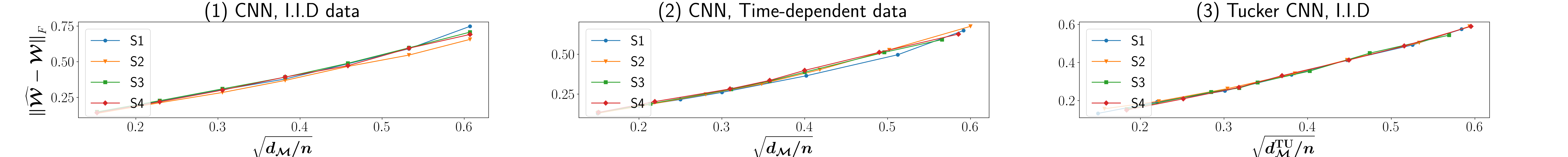}
	\caption{(1)-(2) are the experiment results for Theorem 1 with independent and dependent inputs, respectively. (3) is the experiment results for Tucker CNN in Theorem 2. }
	\label{fig:theorem1-2}
\end{figure}

Next, we conduct extra experiments for efficient number of kernels for a CP block design. 
This study uses $32\times 32$ inputs with $16$ channels, and we set stride to 1 and pooling sizes to $(5,5)$. 
We generate the orthonormal factor matrices $\{\bm{H}^{(j)}, 1\leq j\leq 4\}$, where $\bm{H}^{(1)}$ is of size $\mathbb{R}^{8\times R}$, $\bm{H}^{(2)}$ is of size $\mathbb{R}^{8\times R}$, $\bm{H}^{(3)}$ is of size $\mathbb{R}^{16\times R}$ and $\bm{H}^{(4)}$ is of size $\mathbb{R}^{K\times R}$ where $R=8$ and $K\in\{8, 16, 24, 32\}$.
If we denote the orthonormal column vectors of $\bm{H}^{(j)}$ by $\bm{h}^{(j)}_r$ where $1\leq r\leq R$ and $1\leq j\leq 4$, the stacked kernel tensor $\cm{A}$ can then be generated as $\cm{A} = \sum_{r=1}^{R}\bm{h}^{(1)}_r\circ\bm{h}^{(2)}_r\circ\bm{h}^{(3)}_r\circ\bm{h}^{(4)}_r$. And it can be seen that $\cm{A}$ has a CP rank of $R=8$. We split the stacked kernel tensor $\cm{A}$ along the kernel dimension to obtain $\{\cm{A}_k, 1\leq k\leq K\}$ and generate the corresponding fully-connected weight tensors $\{\cm{B}_k, 1\leq k\leq K\}$ with standard normal entries. The parameter tensor $\cm{W}$ is hence obtained, and we further normalize it to have unit Frobenius norm to ensure the comparability of estimation errors between different $K$s. 
The block structure in Figure \ref{fig:kernel}(1) is employed to train the network, and it is equivalent to the bottleneck structure with a CP decomposition on $\cm{A}$; see \citet{kossaifi2019factorized}. 
We can see that as $K$ increases, there is more redundancy in the network parameters.
Fifty training sets are generated for each training size, and we stop training when the target function drops by less than $10^{-5}$.
From Figure \ref{fig:kernel}(2), the estimation error increases as $K$ is larger, and the difference is more pronounced when training size is small.

\begin{figure}[ht]
	\centering
	\includegraphics[width=0.7\linewidth]{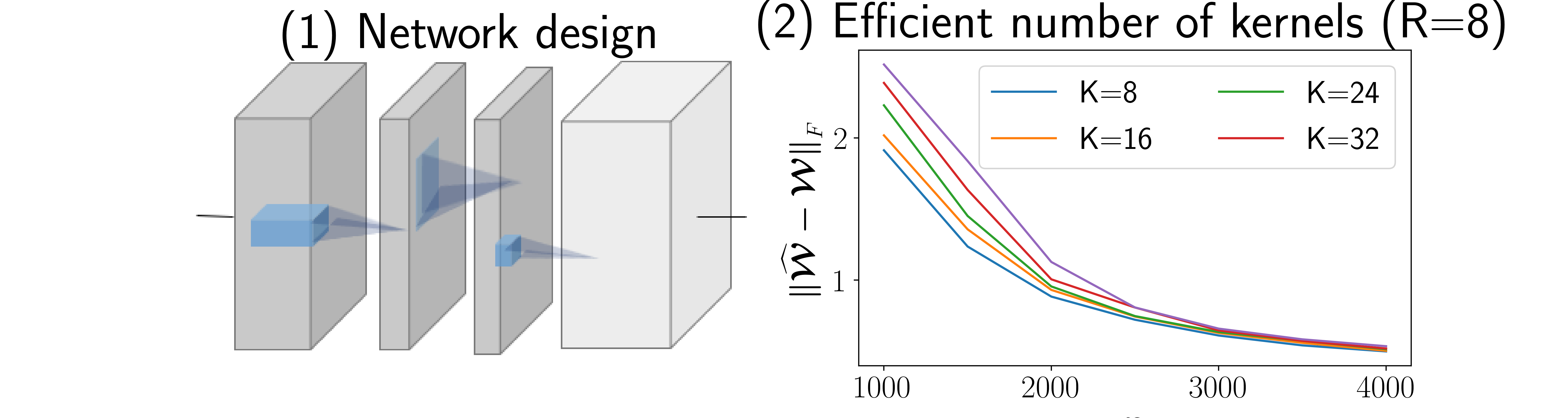}
	\caption{Additional experiments for efficient number of kernels.}
	\label{fig:kernel}
\end{figure}

\newpage

\renewcommand{\thesection}{Appendix C}
\renewcommand{\thelemma}{C.\arabic{lemma}}
\renewcommand{\thesubsection}{C.\arabic{subsection}}
\renewcommand{\theequation}{C.\arabic{equation}}
\setcounter{lemma}{0}
\setcounter{equation}{0}

\section{Theoretical results and technical proofs}\label{app:thm}

\subsection{Proof of Theorem 1}

Denote the sets $\widehat{\mathcal{S}}_K=\{\sum_{k=1}^{K}\cm{B}_k\otimes\cm{A}_k: \cm{A}_k\in\mathbb{R}^{l_1\times l_2\times\cdots\times l_N} \text{ and } \cm{B}_k\in\mathbb{R}^{p_1\times p_2\times\cdots\times p_N} \}$ and $\mathcal{S}_K=\{\cm{W}\in \widehat{\mathcal{S}}_K: \|\cm{W}\|_{\mathrm{F}}=1\}$.  
Let $\widehat{\bm{\Delta}}=\cm{\widehat{W}} -\cm{W}^*$, and then
\[
\frac{1}{n}\sum_{i=1}^{n}(y^i - \langle\cm{Z}^i,\cm{\widehat{W}}\rangle)^2 \leq \frac{1}{n}\sum_{i=1}^{n}(y^i - \langle\cm{Z}^i,\cm{W}^*\rangle)^2,
\]
which implies that
\begin{equation}\label{eq:basic}
	\|\widehat{\bm{\Delta}}\|_n^2 \leq \frac{2}{n}\sum_{i=1}^{n}\xi^i\langle\cm{Z}^i,\widehat{\bm{\Delta}}\rangle \leq 2\|\widehat{\bm{\Delta}}\|_{\mathrm{F}}\sup_{\bm{\Delta}\in \mathcal{S}_{2K}}\frac{1}{n}\sum_{i=1}^{n}\xi^i\langle\cm{Z}^i,\bm{\Delta}\rangle,
\end{equation}
where $\norm{\bm{\Delta}}_n^2 := {n}^{-1}\sum_{i=1}^{n}(\langle\cm{Z}^i,\bm{\Delta}\rangle)^2$ is the empirical norm with respect to $\bm{\Delta}$, and $\widehat{\bm{\Delta}}\in \widehat{\mathcal{S}}_{2K}$.

Consider a $\varepsilon$-net $\bar{\mathcal{S}}_{2K}$, with the cardinality of $\mathcal{N}(2K,\varepsilon)$, for the set $\mathcal{S}_{2K}$. For any $\bm{\Delta}\in\mathcal{S}_{2K}$, there exists a $\bar{\bm{\Delta}}_j\in\bar{\mathcal{S}}_{2K}$ such that $\|\bm{\Delta}-\bar{\bm{\Delta}}_j\|_{\mathrm{F}}\leq \varepsilon$.
Note that $\bm{\Delta}-\bar{\bm{\Delta}}_j\in\widehat{\mathcal{S}}_{4K}$ and, from Lemma \ref{lemma:partition}(a), we can further find  $\bm{\Delta}_1,\bm{\Delta}_2\in\widehat{\mathcal{S}}_{2K}$ such that $\langle\bm{\Delta}_1,\bm{\Delta}_2\rangle=0$ and $\bm{\Delta}-\bar{\bm{\Delta}}_j =\bm{\Delta}_1+\bm{\Delta}_2$. It then holds that
$\|\bm{\Delta}_1\|_{\mathrm{F}}+\|\bm{\Delta}_2\|_{\mathrm{F}}\leq \sqrt{2}\|\bm{\Delta}-\bar{\bm{\Delta}}_j\|_{\mathrm{F}}\leq \sqrt{2}\varepsilon$ since $\|\bm{\Delta}-\bar{\bm{\Delta}}_j\|_{\mathrm{F}}^2=\|\bm{\Delta}_1\|_{\mathrm{F}}^2+\|\bm{\Delta}_2\|_{\mathrm{F}}^2$.
As a result,
\begin{align*}
	\frac{1}{n}\sum_{i=1}^{n}\xi^i\langle\cm{Z}^i,\bm{\Delta}\rangle&=\frac{1}{n}\sum_{i=1}^{n}\xi^i\langle\cm{Z}^i,\bar{\bm{\Delta}}_j\rangle+\frac{1}{n}\sum_{i=1}^{n}\xi^i\langle\cm{Z}^i,\bm{\Delta}_1\rangle +\frac{1}{n}\sum_{i=1}^{n}\xi^i\langle\cm{Z}^i,\bm{\Delta}_2\rangle\\
	&\leq\max_{1\leq j\leq \mathcal{N}(2K,\varepsilon)}\frac{1}{n}\sum_{i=1}^{n}\xi^i\langle\cm{Z}^i,\bar{\bm{\Delta}}_j\rangle+\sqrt{2}\varepsilon \sup_{\bm{\Delta}\in \mathcal{S}_{2K}}\frac{1}{n}\sum_{i=1}^{n}\xi^i\langle\cm{Z}^i,\bm{\Delta}\rangle,
\end{align*}
which leads to
\begin{equation}\label{eq:sup2}
	\begin{split}
		&\sup_{\bm{\Delta}\in \mathcal{S}_{2K}}\frac{1}{n}\sum_{i=1}^{n}\xi^i\langle\cm{Z}^i,\bm{\Delta}\rangle \leq (1-\sqrt{2}\varepsilon)^{-1}\max_{1\leq j\leq \mathcal{N}(2K,\varepsilon)}\frac{1}{n}\sum_{i=1}^{n}\xi^i\langle\cm{Z}^i,\bar{\bm{\Delta}}_j\rangle.
	\end{split}
\end{equation}

Note that, from Lemma \ref{lemma:partition}(b), $\log\mathcal{N}(2K,\varepsilon)\leq 2d_{\mathcal{M}}\log(9/\varepsilon)$, where $d_{\mathcal{M}}=K(P+L+1)$.
Let $\varepsilon=(2\sqrt{2})^{-1}$, and then $8-2\log(9/\varepsilon)>1$.
As a result, by (\ref{eq:sup2}) and Lemma \ref{lemma:concentration},
\begin{align}
	\begin{split}\label{eq:thm}
		&\mathbb{P}\Big\{\sup_{\bm{\Delta}\in \mathcal{S}_{2K}}\frac{1}{n}\sum_{i=1}^{n}\xi^i\langle\cm{Z}^i,\bm{\Delta}\rangle \geq 8\sigma\sqrt{\alpha_{\mathrm{RSM}}} \sqrt{\frac{d_{\mathcal{M}}}{n}},\sup_{\bm{\Delta}\in \mathcal{S}_{2K}}\norm{\bm{\Delta}}_n^2 \leq \alpha_{\mathrm{RSM}}\Big\}\\
		&\leq \mathbb{P} \Big\{\max_{1\leq j\leq \mathcal{N}(2K,\varepsilon)}\frac{1}{n}\sum_{i=1}^{n}\xi^i\langle\cm{Z}^i,\bar{\bm{\Delta}}_j\rangle \geq 4\sigma \sqrt{\alpha_{\mathrm{RSM}}} \sqrt{\frac{d_{\mathcal{M}}}{n}},\sup_{\bm{\Delta}\in \mathcal{S}_{2K}}\norm{\bm{\Delta}}_n^2 \leq \alpha_{\mathrm{RSM}}\Big\}\\
		&\leq \sum_{j=1}^{\mathcal{N}(2K,\varepsilon)} \mathbb{P} \Big(\frac{1}{n}\sum_{i=1}^{n}\xi^i\langle\cm{Z}^i,\bar{\bm{\Delta}}_j\rangle \geq 4\sigma\sqrt{\alpha_{\mathrm{RSM}}} \sqrt{\frac{ d_{\mathcal{M}}}{n}},\norm{\bar{\bm{\Delta}}_j}_n^2 \leq \alpha_{\mathrm{RSM}}\Big)\\
		&\leq \exp\{- d_{\mathcal{M}}\}.
	\end{split}
\end{align}
Note that
\begin{align*}
	&\mathbb{P} \left\{\sup_{\bm{\Delta}\in \mathcal{S}_{2K}}\frac{1}{n}\sum_{i=1}^{n}\xi^i\langle\cm{Z}^i,\bm{\Delta}\rangle \geq 8\sigma \sqrt{\alpha_{\mathrm{RSM}}}\sqrt{\frac{d_{\mathcal{M}}}{n}}\right\} \\
	&\leq 
	\mathbb{P} \Big\{\sup_{\bm{\Delta}\in \mathcal{S}_{2K}}\frac{1}{n}\sum_{i=1}^{n}\xi^i\langle\cm{Z}^i,\bm{\Delta}\rangle \geq 8\sigma \sqrt{\alpha_{\mathrm{RSM}}}\sqrt{\frac{d_{\mathcal{M}}}{n}},\sup_{\bm{\Delta}\in \mathcal{S}_{2K}}\norm{\bm{\Delta}}_n^2 \leq \alpha_{\mathrm{RSM}}\Big\}+\mathbb{P}\left( \sup_{\bm{\Delta}\in \mathcal{S}_{2K}}\norm{\bm{\Delta}}_n^2 \geq \alpha_{\mathrm{RSM}}\right).
\end{align*}
From (\ref{eq:thm}) and Lemma \ref{lemma:RSC}, we then have that, with probability $1-4\exp\{-[c_{\mathrm{H}}(0.25\kappa_L/\kappa_U)^2n-9d_{\mathcal{M}}]\}-\exp\{ -d_{\mathcal{M}}\}$, 
\[
\sup_{\bm{\Delta}\in \mathcal{S}_{2K}}\frac{1}{n}\sum_{i=1}^{n}\xi^i\langle\cm{Z}^i,\bm{\Delta}\rangle \leq 8\sigma\sqrt{\alpha_{\mathrm{RSM}}} \sqrt{\frac{d_{\mathcal{M}}}{n}}
\]
and
\[
\norm{\bm{\Delta}}_n^2\geq \alpha_{\mathrm{RSC}}\norm{\bm{\Delta}}_\mathrm{F}^2,
\]
which, together with (\ref{eq:basic}), leads to
\[
\|\widehat{\bm{\Delta}}\|_{\mathrm{F}} \leq \frac{16\sigma\sqrt{\alpha_{\mathrm{RSM}}}}{\alpha_{\mathrm{RSC}}}\sqrt{\frac{d_{\mathcal{M}}}{n}}.
\]

Given a test sample $(\cm{X},y)$, and let $\cm{Z} = \cm{X}\times_1\bm{U}^{(1)\prime}_\mathcal{F}\times_2\bm{U}^{(2)\prime}_\mathcal{F}\times_3\cdots\times_N\bm{U}^{(N)\prime}_\mathcal{F}$. 
It holds that
\[
\cm{E}^2(\cm{\widehat{W}}) = \mathbb{E}_{\cm{X}} |\langle\cm{X}, \cm{\widehat{W}}_X\rangle -  \langle\cm{X}, \cm{W}^*_X\rangle|^2 =\mathbb{E}|\langle\cm{Z}, \cm{\widehat{W}}-\cm{W}^*\rangle|^2 
\]
and by Assumptions 1\&3, we have
\[
\cm{E}(\cm{\widehat{W}}) \leq \sqrt{\kappa_U}\|\widehat{\bm{\Delta}}\|_{\mathrm{F}}.
\]
This accomplishes the proof.

\subsection{Proof of Theorem 2} \label{pf:thm2}

Denote the sets $\widehat{\mathcal{S}}^{\mathrm{TU}}(R_1,\cdots,R_{N}, R)=\{\sum_{k=1}^{K}\cm{B}_k\otimes\cm{A}_k:$ the stacked kernel $\cm{A}_\mathrm{stack}\in\mathbb{R}^{l_1\times l_2\times\cdots\times l_N\times K}$ has the ranks of $(R_1,...,R_N, R)$ and $\cm{B}_k\in\mathbb{R}^{p_1\times p_2\times\cdots\times p_N} \}$,
$\widehat{\mathcal{S}}_{2K}^{\mathrm{TU}}=\{\cm{W}_1+\cm{W}_2: \cm{W}_1,\cm{W}_2 \in \widehat{\mathcal{S}}^{\mathrm{TU}}(R_1,\cdots,R_N,R) \}$,
and $\mathcal{S}_{2K}^\mathrm{TU}  = \{\cm{W}\in\widehat{\mathcal{S}}_{2K}^\mathrm{TU}: \|\cm{W}\|_\mathrm{F} = 1\}$.
Note that we have $\cm{\widehat{W}}_{\mathrm{TU}},\cm{W}^*\in \widehat{\mathcal{S}}^{\mathrm{TU}}(R_1,\cdots,R_N,R)$, and  $\widehat{\bm{\Delta}}=\cm{\widehat{W}}_{\mathrm{TU}} -\cm{W}^*\in \widehat{\mathcal{S}}_{2K}^{\mathrm{TU}}$. 

We first consider a $\varepsilon$-net for $\mathcal{S}_{2K}^\mathrm{TU}$.
For each $1\leq k\leq K$, Let $\bm{b}_k=(b_{k1},\ldots,b_{kP})^{\prime}=\vectorize(\cm{B}_k)$, and we can rearrange $\cm{B}_k\otimes\cm{A}_k$ into the form of $\cm{A}_k\circ \bm{b}_k$, which is a tensor of size $\mathbb{R}^{l_1\times l_2\times\cdots \times l_N\times P}$, where $P=p_1p_2\cdots p_N$. Denote $\bm B=(b_{kj})\in \mathbb{R}^{P\times K}$, and it holds that $$\sum_{k=1}^K\cm{A}_k\circ \bm{b}_k=\cm{A}\times_{N+1}\bm B= \cm{G}\times_1 \bm{H}^{(1)}\times_2 \bm{H}^{(2)}\cdots\times_{N+1}\widetilde{\bm{B}},$$
which is a tensor with the size of $\mathbb{R}^{l_1\times l_2\times\cdots \times l_N\times P}$ and the ranks of $(R_1,...,R_N, R)$ where $\widetilde{\bm{B}} = \bm{B}\bm{H}^{(N+1)}\in\mathbb{R}^{P\times R}$.
Essentially, in this step, we rewrite the model into one with $R_{k+1}$ kernels instead. Specifically, we now have $\cm{W} = \sum_{r=1}^{R}\cm{\widetilde{A}}_r\otimes\cm{\widetilde{B}}_r$, where $\cm{\widetilde{A}}_r = \cm{G}_r\times_1 \bm{H}^{(1)}\times_2 \bm{H}^{(2)}\cdots\times_{N}\bm{H}^{(N)}$ with $\cm{G}_r = \cm{G}(:,:,\cdots,r)$ and the $r$-th column of $\widetilde{\bm{B}}$ is the vectorization of $\cm{\widetilde{B}}_r$ for $1\leq r\leq R$.

As a result, $\widehat{\mathcal{S}}_{2K}^\mathrm{TU}$ consists of tensors with the ranks of $(2R_1,...,2R_N, 2R)$ at most. 

Denote 
$\mathcal{S}_{\mathrm{Tucker}}(r_1,\cdots,r_N,r_{N+1})=\{\cm{T}\in \mathbb{R}^{l_1\times \cdots\times l_N\times l_{N+1}}: \|\cm{T}\|_{\mathrm{F}}=1,$ $\cm{T}$ has the Tucker ranks of $(r_1,\cdots,r_N,r_{N+1})\}$, where $l_{N+1}=P = p_1p_2\cdots p_N$. 
Then the $\varepsilon$-covering number for $\mathcal{S}_{2K}^\mathrm{TU}$ satisfies
\[
|\mathcal{S}_{2K}^\mathrm{TU}|\leq |\mathcal{S}_{\mathrm{Tucker}}(2R_1,\cdots,2R_N,2R)|.
\]

For each $\cm{T}\in \mathcal{S}_{\mathrm{Tucker}}(r_1,\cdots,r_N,r_{N+1})$, we have
\begin{equation*}
	\cm{T} = \cm{G}\times_1 \bm{U}^{(1)}\times_2\cdots\times_{N+1} \bm{U}^{(N+1)},
\end{equation*}
where $\cm{G}\in\mathbb{R}^{r_1\times\cdots \times r_N\times r_{N+1}}$ with $\|\cm{G}\|_{\mathrm{F}}=1$, and $\bm{U}^{(i)}\in\mathbb{R}^{l_i\times r_i}$ with $1\leq i\leq N+1$ are orthonormal matrices.
We now construct an $\varepsilon$-net for $\mathcal{S}_{\mathrm{Tucker}}(r_1,\cdots,r_N,r_{N+1})$ by covering the sets of $\cm{G}$ and all $\bm{U}^{(i)}$s, and the proof hinges on the covering number of low-multilinear-rank tensors in \cite{wang2019compact}. 
Treating $\cm{G}$ as $\prod_{j=1}^{N+1}r_j$-dimensional vector with $\|\cm{G}\|_{\mathrm{F}}=1$, we can find an $\varepsilon/(N+2)$-net for it, denoted by $\bar{G}$, with the cardinality of $|\bar{G}|\leq(3(N+2)/\varepsilon)^{\prod_{j=1}^{N+1}r_j}$. 

Next, let $O_{n,r}=\{\bm{U}\in\mathbb{R}^{n\times r}: \bm{U}^\top\bm{U}=I_{r}\}$. To cover $O_{n,r}$, it is beneficial to use the $\|\cdot\|_{1,2}$ norm, defined as
\begin{equation*}
	\|\bm{X}\|_{1,2}=\max_{i}\|\bm{X}_i\|_2,
\end{equation*}
where $\bm{X}_i$ denotes the $i\ts{th}$ column of $\bm{X}$. Let $Q_{n,r}=\{\bm{X}\in\mathbb{R}^{n\times r}:\|\bm{X}\|_{1,2}\leq1\}$. One can easily check that $O_{n,r}\subset Q_{n,r}$, and then an $\varepsilon/(N+2)$-net $\bar{O}_{n,r}$ for $O_{n,r}$ has the cardinality of $|\bar{O}_{n,r}|\leq(3(N+2)/\varepsilon)^{nr}$.
Denote $\bar{\mathcal{S}}_{\mathrm{Tucker}}(r_1,\cdots,r_N,r_{N+1})=\{\cm{G}\times_1 \bm{U}^{(1)}\times_2\cdots\times_{N+1} \bm{U}^{(N+1)}:~\cm{{G}}\in\bar{G},~\bm{{U}}^{(i)}\in \bar{O}_{l_i,r_i},~1\leq i \leq N+1\}$. 
By a similar argument presented in Lemma A.1 of \cite{wang2019compact}, we can show that $\bar{\mathcal{S}}_{\mathrm{Tucker}}(r_1,\cdots,r_N,r_{N+1})$ is a $\varepsilon$-net
for the set ${\mathcal{S}}_{\mathrm{Tucker}}(r_1,\cdots,r_N,r_{N+1})$ with the cardinality of
\[
\left(\frac{3N+6}{\varepsilon}\right)^{\prod_{j=1}^{N+1}r_j+\sum_{j=1}^{N+1}l_jr_j}.
\]
where $l_{N+1}=P$. Thus, the $\varepsilon$-covering number of $\mathcal{S}_{2K}^\mathrm{TU}$ is
\begin{align*}
	&\mathcal{N}^{\mathrm{TU}}(\varepsilon)= |\mathcal{S}_{\mathrm{Tucker}}(2R_1,\cdots,2R_N,2R)|\\ &\hspace{8mm}\leq\left(\frac{3N+6}{\varepsilon}\right)^{2^{N+1}R\prod_{j=1}^{N}R_j+2\sum_{i=1}^{N}l_iR_i+2RP}.
\end{align*}

Let $\varepsilon = 1/2$. It then holds that $\log(3(N+2)/\varepsilon) < 3\log(N+2)$ and $\log \mathcal{N}^{\mathrm{TU}}(\varepsilon)\leq 2^{N+1}d_{\mathcal{M}}^\mathrm{TU}\log[(3N+6)/\varepsilon]$, where $d_{\mathcal{M}}^\mathrm{TU}=R\prod_{j=1}^{N}R_j+\sum_{i=1}^{N}l_iR_i+RP$. By a method similar to the proof of Lemma 2, we can show that 
\begin{equation}
	\label{eq:thm3-RSM}
	\mathbb{P}\left\{\sup_{\bm{\Delta}\in\mathcal{S}_{2K}^\mathrm{TU}}
	\norm{\bm{\Delta}}_n^2 
	\geq\alpha_{\mathrm{RSM}}\right\}
	\leq2\exp\left\{-c_{\mathrm{H}}n\left(\frac{\kappa_L}{4\kappa_U}\right)^2+3c_Nd_{\mathcal{M}}^\mathrm{TU}\right\},
\end{equation}
and 
\begin{equation}
	\label{eq:thm3-RSC}
	\mathbb{P}\left\{\inf_{\bm{\Delta}\in\mathcal{S}_{2K}^\mathrm{TU}}
	\norm{\bm{\Delta}}_n^2 
	\leq \alpha_{\mathrm{RSC}}\right\}
	\leq 2\exp\left\{-c_{\mathrm{H}}n\left(\frac{\kappa_L}{4\kappa_U}\right)^2+3c_Nd_{\mathcal{M}}^\mathrm{TU}\right\}.
\end{equation}
where $c_N = 2^{N+1}\log(N+2)$.
Moreover, similar to \eqref{eq:thm}, by applying Lemma \ref{lemma:concentration}, we can show that
\begin{align}
	\begin{split}\label{eq:thm-3}
		\mathbb{P} &\Big\{\sup_{\bm{\Delta}\in \mathcal{S}_{2K}^\mathrm{TU}}\frac{1}{n}\sum_{i=1}^{n}\xi^i\langle\cm{Z}^i,\bm{\Delta}\rangle \geq 8\sigma\sqrt{\alpha_{\mathrm{RSM}}} \sqrt{\frac{c_Nd_{\mathcal{M}}^\mathrm{TU}}{n}},\sup_{\bm{\Delta}\in \mathcal{S}_{2K}^\mathrm{TU}}\norm{\bm{\Delta}}_n^2 \leq \alpha_{\mathrm{RSM}}\Big\}\\
		&\leq \mathbb{P} \Big\{\max_{1\leq j\leq \mathcal{N}^\mathrm{TU}(\varepsilon)}\frac{1}{n}\sum_{i=1}^{n}\xi^i\langle\cm{Z}^i,\bar{\bm{\Delta}}_j\rangle \geq 4\sigma \sqrt{\alpha_{\mathrm{RSM}}} \sqrt{\frac{c_Nd_{\mathcal{M}}^\mathrm{TU}}{n}},\sup_{\bm{\Delta}\in \mathcal{S}_{2K}^\mathrm{TU}}\norm{\bm{\Delta}}_n^2 \leq \alpha_{\mathrm{RSM}}\Big\}\\
		&\leq \sum_{j=1}^{\mathcal{N}^\mathrm{TU}(\varepsilon)} \mathbb{P} \Big\{\frac{1}{n}\sum_{i=1}^{n}\xi^i\langle\cm{Z}^i,\bar{\bm{\Delta}}_j\rangle \geq 4\sigma\sqrt{\alpha_{\mathrm{RSM}}} \sqrt{\frac{ c_Nd_{\mathcal{M}}^\mathrm{TU}}{n}},\norm{\bar{\bm{\Delta}}_j}_n^2 \leq \alpha_{\mathrm{RSM}}\Big\}\\
		&\leq \exp\{- d_{\mathcal{M}}^\mathrm{TU}\},
	\end{split}
\end{align}
where $\varepsilon = (2\sqrt{2})^{-1}$ and $2^{N+4}\log(N+2) - 2^{N+1}\log(3(N+1)/\varepsilon) > 1$.
By a method similar to the proof of Theorem 1, we can show that, with probability $1-4\exp\{-[c_{\mathrm{H}}(0.25\kappa_L/\kappa_U)^2n-3c_N\log(N+2)d_{\mathcal{M}}^\mathrm{TU}]\}-\exp\{ -d_{\mathcal{M}}^\mathrm{TU}\}$, 
\[
\sup_{\bm{\Delta}\in \mathcal{S}_{2K}^{\mathrm{TU}}}\frac{1}{n}\sum_{i=1}^{n}\xi^i\langle\cm{Z}^i,\bm{\Delta}\rangle \leq 8\sigma\sqrt{\alpha_{\mathrm{RSM}}}\sqrt{\frac{c_Nd_{\mathcal{M}}^\mathrm{TU}}{n}},
\]
and
\[
\norm{\bm{\Delta}}_n^2\geq \alpha_{\mathrm{RSC}}\norm{\bm{\Delta}}_\mathrm{F}^2,
\]
which, together with \eqref{eq:basic}, leads to
\begin{equation*}
	\begin{split}
		&\|\widehat{\bm{\Delta}}\|_{\mathrm{F}} \leq \frac{16\sigma\sqrt{\alpha_{\mathrm{RSM}}}}{\alpha_{\mathrm{RSC}}} \sqrt{\frac{c_Nd_{\mathcal{M}}^\mathrm{TU}}{n}},\text{ and}\\		
		&\cm{E}(\cm{\widehat{W}}) \leq \sqrt{\kappa_U}\|\widehat{\bm{\Delta}}\|_{\mathrm{F}}.
	\end{split}
\end{equation*}
We accomplished the proof.
\subsection{Proofs of Corollary 1}
When the kernel tensor $\cm{A}$ has a Tucker decomposition form of $\cm{A}=\cm{G}\times_1\bm{H}^{(1)}\times_2\bm{H}^{(2)}\times_3\cdots\times_{N+1}\bm{H}^{(N+1)}$, and the multilinear rank is $(R_1,R_2,\cdots,R_N,R)$.
Let $\cm{H} = \cm{G}\times_1\bm{H}^{(1)}\times_2\bm{H}^{(2)}\times_3\cdots\times_{N}\bm{H}^{(N)}$, which is a tensor of size $\mathbb{R}^{l_1\times\cdots\times l_N\times R}$. 
The mode-$(N+1)$ unfolding of $\cm{H}$ is a matrix with $R$ row vectors, each of size $\mathbb{R}^{L}$, and we denote them by $\bm{g}_1,\cdots, \bm{g}_{R_{N}}, \bm{g}_{R}$.
Fold $\bm{g}_r$s back into tensors, i.e. $\bm{g}_r = \vectorize(\cm{\widetilde{H}}_r)$, where $\cm{\widetilde{H}}_r\in\mathbb{R}^{l_1\times\cdots\times l_N}$ and $1\leq r\leq R$.
Moreover, let $\bm{H}^{(N+1)} = (\bm{h}_1^{(N+1)},\cdots,\bm{h}_K^{(N+1)})^\prime$, where $\bm{h}_k^{(N+1)}$ is a vector of size $\mathbb{R}^{R}$ and we denote its $r\ts{th}$ entry as $h_{k,r}^{(N+1)}$, for $1\leq r \leq R$. It can be verified that, for each $1\leq k \leq K$,
\begin{align*}
	\cm{A}_k &= \cm{G}\times_1\bm{H}^{(1)}\times_2\bm{H}^{(2)}\cdots\times_{N}\bm{H}^{(N)}\bar{\times}_{N+1}\bm{\bm{h}}_k^{(N+1)}\\
	&= \cm{H}\bar{\times}_{N+1}\bm{\bm{h}}_k^{(N+1)},
\end{align*}
and hence,
\begin{align*}
	\sum_{k=1}^{K}\cm{B}_k\otimes\cm{A}_k 
	&= \sum_{r=1}^{R}
	\sum_{k=1}^{K}\cm{B}_k\otimes\left(\cm{\widetilde{H}}_r\cdot h_{k,r}^{(N+1)}\right) \\
	&= \sum_{r=1}^{R}
	\left(\sum_{k=1}^{K}h_{k,r}^{(N+1)}\cm{B}_k\right)\otimes\cm{\widetilde{H}}_r.
\end{align*}
By letting $\cm{\widetilde{B}}_r =\sum_{k=1}^{K}h_{k,r}^{(N+1)}\cm{B}_k$ and $\cm{\widetilde{A}}_r = \cm{\widetilde{H}}_r$ for $1\leq r \leq R$, we can reformulate the model into
\begin{equation*}
	y^i = \langle \sum_{r=1}^{R}\cm{\widetilde{B}}_r\otimes\cm{\widetilde{A}}_r, \cm{Z}^i \rangle + \xi^i ,
\end{equation*}

and the proof of (a) is then accomplished.

Suppose that the kernel tensor $\cm{A}$ has a CP decomposition form of $\cm{A} = \sum_{r=1}^{R}\alpha_r\bm{h}^{(1)}_r\circ \bm{h}^{(2)}_r\circ\cdots\circ\bm{h}^{(N+1)}_r$, where $\bm{h}^{(j)}_r\in \mathbb{R}^{l_j}$ and $\|\bm{h}^{(j)}_r\|=1$ for all $1\leq j\leq N$.
Moreover, $\bm{h}^{(N+1)}_r=(h_{r,1}^{(N+1)},...,h_{r,K}^{(N+1)})^{\prime}\in \mathbb{R}^K$ and $\|\bm{h}^{(N+1)}_r\|=1$.
Note that
\begin{equation*}
	\cm{A}_k = \sum_{r=1}^{R}\alpha_rh_{r,k}^{(N+1)}\bm{h}^{(1)}_r\circ \bm{h}^{(2)}_r\circ\cdots\circ\bm{h}^{(N)}_r,
\end{equation*}
for all $1\leq k\leq K$, and hence
\begin{equation*}
	\begin{split}
		\sum_{k=1}^{K}\cm{B}_k\otimes\cm{A}_k 
		&= \sum_{r=1}^{R}\sum_{k=1}^{K}\cm{B}_k\otimes\left(\alpha_rh_{r,k}^{(N+1)}\bm{h}^{(1)}_r\circ \bm{h}^{(2)}_r\circ\cdots\circ\bm{h}^{(N)}_r \right)\\
		&= \sum_{r=1}^{R}\left(h_{r,k}^{(N+1)}\cm{B}_k\right)\otimes\left(\alpha_r\bm{h}^{(1)}_r\circ \bm{h}^{(2)}_r\circ\cdots\circ\bm{h}^{(N)}_r\right).
	\end{split}
\end{equation*}
By letting $\cm{\widetilde{B}}_r = h_{r,k}^{(N+1)}\cm{B}_k$ and $\cm{\widetilde{A}}_r = \alpha_r\bm{h}^{(1)}_r\circ \bm{h}^{(2)}_r\circ\cdots\circ\bm{h}^{(N)}_r$ for all $1\leq r \leq R$, we can reparameterize the model into
\begin{equation*}
	y^i = \langle \sum_{r=1}^{R}\cm{\widetilde{B}}_r\otimes\cm{\widetilde{A}}_r, \cm{Z}^i \rangle + \xi^i ,
\end{equation*}
and the proof of (b) is then accomplished.

\subsection{Lemmas Used in the Proofs of Theorems}

\begin{lemma}(Partition and covering number of restricted spaces)\label{lemma:partition}
	Consider $\widehat{\mathcal{S}}_K=\{\sum_{k=1}^{K}\cm{B}_k\otimes\cm{A}_k: \cm{A}_k\in\mathbb{R}^{l_1\times l_2\times\cdots\times l_N} \text{ and } \cm{B}_k\in\mathbb{R}^{p_1\times p_2\times\cdots\times p_N} \}$ and $\mathcal{S}_K=\{\cm{W}\in \widehat{\mathcal{S}}_K: \|\cm{W}\|_{\mathrm{F}}=1\}$.
	\begin{itemize}
		\item [(a)] For any $\bm{\Delta}\in \widehat{\mathcal{S}}_{2K}$, there exist $\cm{W}_1,\cm{W}_2\in\widehat{\mathcal{S}}_{K}$ such that $\bm{\Delta} = \cm{W}_1+\cm{W}_2$ and $\langle\cm{W}_1,\cm{W}_2\rangle=0$.
		\item [(b)]The $\varepsilon$-covering number of the set $\mathcal{S}_{K}$ is 
		$$\mathcal{N}(K,\varepsilon) \leq (9/\varepsilon)^{K(L+P+1)},$$
		where $L=l_1l_2\cdots l_N$ and $P=p_1p_2\cdots p_N$.	
	\end{itemize}
\end{lemma}

\begin{proof}
	For each $K$, we first define a map $T_K: \widehat{\mathcal{S}}_K \rightarrow \mathbb{R}^{P\times L}$. For any $\cm{W}=\sum_{k=1}^{K}\cm{B}_k\otimes\cm{A}_k \in \widehat{\mathcal{S}}_K$, we define that $T_K(\cm{W})= \sum_{k=1}^{K}\vectorize{(\cm{B}_k)}\vectorize(\cm{A}_k)^\prime\in \mathbb{R}^{P\times L}$, which has the rank of at most $K$.
	
	For any $\bm{\Delta}\in \widehat{\mathcal{S}}_{2K}$, the rank of matrix $T_K(\bm{\Delta})$ is at most $2K$. As shown in Figure \ref{fig:partition}, we can split the singular value decomposition (SVD) of $T_K(\bm{\Delta})$ into two parts, and it can be verified that $T_K(\bm{\Delta}) = \bm{U}_1\bm{\Lambda}_1{\bm{V}_1}^{\prime}+\bm{U}_2\bm{\Lambda}_2{\bm{V}_2}^{\prime}$ and $\langle\bm{U}_1\bm{\Lambda}_1{\bm{V}_1}^{\prime},\bm{U}_2\bm{\Lambda}_2{\bm{V}_2}^{\prime}\rangle=0$.
	Denote $\bm{U}_i=(\bm{u}_1^{(i)},\bm{u}_2^{(i)},\ldots,\bm{u}_K^{(i)})$ and $\bm{V}_i\bm{\Lambda}_i^{\prime}=(\bm{v}_1^{(i)},\bm{v}_2^{(i)},\ldots,\bm{v}_K^{(i)})$ for $i=1$ and 2. We then fold $\bm{u}_j^{(i)}$s and $\bm{v}_j^{(i)}$s into tensors, i.e. $\bm{u}_k^{(i)}=\vectorize(\cm{B}_k^{(i)})$ and $\bm{v}_k^{(i)}=\vectorize(\cm{A}_k^{(i)})$ with $1\leq k\leq K$.
	Let $\cm{W}_1=\sum_{k=1}^{K}\cm{B}_k^{(1)}\otimes\cm{A}_k^{(1)}$ and $\cm{W}_2=\sum_{k=1}^{K}\cm{B}_k^{(2)}\otimes\cm{A}_k^{(2)}$.
	It then can be verified that $\cm{W}_1,\cm{W}_2\in\widehat{\mathcal{S}}_{K}$, $\bm{\Delta} = \cm{W}_1+\cm{W}_2$ and $\langle\cm{W}_1,\cm{W}_2\rangle=0$. Thus, we accomplish the proof of (a).
	
	\begin{figure}[ht]
		\begin{center}
			\centerline{\includegraphics[width=0.4\linewidth]{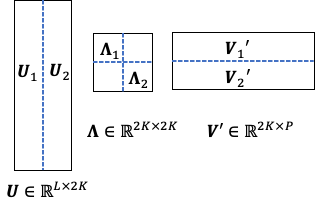}}
			\caption{Splitting matrix $T(\bm{\Delta})$ based on its singular value decomposition.}
			\label{fig:partition}
		\end{center}
		\vskip -0.3in
	\end{figure}
	
	Denote by $\mathcal{S}_{\mathrm{matrix}}\subset \mathbb{R}^{P\times L}$ the set of matrices with unit Frobenius norm and rank at most $K$. Note that $T(\mathcal{S}_K)\subset \mathcal{S}_{\mathrm{matrix}}$, while the $\varepsilon$-covering number for $\mathcal{S}_{\mathrm{matrix}}$ is 
	$$|\mathcal{S}_{\mathrm{matrix}}| \leq (9/\varepsilon)^{K(L+P+1)};$$
	see \cite{candes2011tight}. This accomplishes the proof of (b).
	
\end{proof}

\begin{lemma}(Restricted strong convexity and smoothness)\label{lemma:RSC}
	Suppose that $n\gtrsim(\kappa_U/\kappa_L)^2d_{\mathcal{M}}$. It holds that, with probability at least $1-4\exp\{-[c_{\mathrm{H}}(0.25\kappa_L/\kappa_U)^2n-9d_{\mathcal{M}}]\}$,
	\[
	\alpha_{\mathrm{RSC}}\norm{\bm{\Delta}}_\mathrm{F}^2\leq \frac{1}{n}\sum_{i=1}^{n}(\langle\cm{Z}^i,\bm{\Delta}\rangle)^2 \leq\alpha_{\mathrm{RSM}}\norm{\bm{\Delta}}_\mathrm{F}^2,
	\]
	for all $\bm{\Delta}\in \widehat{\mathcal{S}}_{2K}$,
	where $\widehat{\mathcal{S}}_{2K}$ is defined in Lemma \ref{lemma:partition}, $c_{\mathrm{H}}$ is a positive constant from the Hanson-Wright inequality, $\alpha_{\mathrm{RSC}} =\kappa_L/2$ and $\alpha_{\mathrm{RSM}} =3\kappa_U/2$.	
\end{lemma}
\begin{proof}
	It is sufficient to show that $\sup_{\bm{\Delta}\in \mathcal{S}_{2K}}\norm{\bm{\Delta}}_n^2 \leq\alpha_{\mathrm{RSM}}$ and $\inf_{\bm{\Delta}\in \mathcal{S}_{2K}}\norm{\bm{\Delta}}_n^2 \geq\alpha_{\mathrm{RSC}}$ hold with a high probability, where $\norm{\bm{\Delta}}_n^2 := {n}^{-1}\sum_{i=1}^{n}(\langle\cm{Z}^i,\bm{\Delta}\rangle)^2$ is the empirical norm, and $\mathcal{S}_{2K}$ is defined in Lemma \ref{lemma:partition}.
	Without confusion, in this proof, we will also use the notation of $\bm{\Delta}$ for its vectorized version, $\vectorize(\bm{\Delta})$.
	Note that
	\begin{equation}
		\begin{split}\label{eq:quadraticform}
			\norm{\bm{\Delta}}_n^2 
			&=\frac{1}{n}\sum_{i=1}^{n}(\langle\cm{Z}^i,\bm{\Delta}\rangle)^2
			= \frac{1}{n}\sum_{i=1}^{n}{\bm{z}^{i\prime}}\bm{\Delta}{\bm{\Delta}}^\prime\bm{z}^i
			\\
			&=\frac{1}{n}\sum_{i=1}^{n}{\bm{x}^{i\prime}}\bm{U}_G\bm{\Delta}{\bm{\Delta}^\prime}\bm{U}_G^\prime\bm{x}^i\\
			&=\frac{1}{n}{\bar{\bm{x}}^\prime}\Big\{\bm{I}_n\otimes(\bm{U}_G\bm{\Delta}{\bm{\Delta}^\prime}\bm{U}^\prime_G)\Big\}\bar{\bm{x}} 
			={\bar{\bm{w}}}^\prime \bm{Q} \bar{\bm{w}},
		\end{split}
	\end{equation}
	where $\bm{Q} =n^{-1}\bm{\Sigma}^{\frac{1}{2}}[\bm{I}_n\otimes(\bm{U}_G\bm{\Delta}{\bm{\Delta}^\prime}\bm{U}^\prime_G)]\bm{\Sigma}^{\frac{1}{2}}$, $\bar{\bm{w}}\in\mathbb{R}^{nD}$ is a random vector with $i.i.d.$ standard normal entries, $\bm{z}^i = \bm{U}_G^{\prime}\bm{x}^i$, $\bm{x}^i = \vectorize(\cm{X}^{i})$, $\bar{\bm{x}} = (\bm{x}^{1\prime},\bm{x}^{2\prime},\ldots,\bm{x}^{n\prime})^{\prime}$, and $\bm{\Sigma}=\mathbb{E}(\bar{\bm{x}}\bar{\bm{x}}^{\prime})\in\mathbb{R}^{nD\times nD}$.
	
	Denote $\bm{w}_G = \bm{U}_G\bm{\Delta}\in\mathbb{R}^{D}$ and $\mathcal{S}^{n-1} = \{\bm{v}\in\mathbb{R}^{n}:\norm{\bm{v}}_2 = 1\}$. Let $\lambda_{\mathrm{max}}(\bm{\Sigma}_{ii})$ be the maximum eigenvalue of $\bm{\Sigma}_{ii} =\mathbb{E}(\vectorize{(\cm{X}^i)}\vectorize{(\cm{X}^i)}^\prime)$ for $1\leq i\leq n$, and it holds that $\lambda_{\mathrm{max}}(\bm{\Sigma}_{ii})\leq C_x$ for all $1\leq i\leq n$. 
	For matrix $\bm{Q}$, we have
	\begin{equation}
		\label{eq: trQ}
		\begin{split}
			\trace(\bm{Q}) &=\frac{1}{n}\trace\left(\bm{\Sigma}^{\frac{1}{2}}[\bm{I}_n\otimes\bm{w}_G][\bm{I}_n\otimes\bm{w}_G^\prime]\bm{\Sigma}^{\frac{1}{2}}\right) \\
			&=\frac{1}{n}\trace\left([\bm{I}_n\otimes\bm{w}_G^\prime]\bm{\Sigma}[\bm{I}_n\otimes\bm{w}_G]\right)
			= \frac{1}{n}\sum_{i=1}^n(\bm{w}_G^\prime\bm{\Sigma}_{ii}\bm{w}_G)\\
			&\leq\frac{1}{n}\sum_{i=1}^n
			\left(\sup_{\bm{w}_G\in \mathcal{S}^{PL-1}}\frac{\bm{w}_G^\prime\bm{\Sigma}_{ii}\bm{w}_G}{\bm{w}_G^\prime\bm{w}_G}\right)
			\left(\sup_{\bm{\Delta}\in \mathcal{S}^{PL-1}}\bm{\Delta}^\prime\bm{U}_G^\prime\bm{U}_G\bm{\Delta}\right)\\
			&\leq \frac{1}{n}\sum_{i=1}^n\lambda_{\mathrm{max}}(\bm{\Sigma}_{ij})C_u
			\leq \kappa_U,
		\end{split}
	\end{equation}
	and similarly, we can show that $\trace(\bm{Q})\geq \kappa_L$,
	where $\kappa_L=c_xc_u$, $\kappa_U=C_xC_u$ and $\kappa_L\leq \kappa_U$.  
	Thus,
	\begin{equation}\label{eq:mean1}
		\kappa_L\leq \mathbb{E}\norm{\bm{\Delta}}_n^2 =\trace(\bm{Q}) \leq \kappa_U.
	\end{equation}
	Moreover, denote $\bm{Q}_1 = n^{-\frac{1}{2}}\bm{\Sigma}^{\frac{1}{2}}[\bm{I}_n\otimes\bm{w}_G]$ and note that $\bm{Q} = \bm{Q}_1\bm{Q}_1^\prime$. To bound the operator norm of $\bm{Q}$, we have
	\begin{equation}
		\label{eq: Q-op}
		\begin{split}
			\norm{\bm{Q}}_{\mathrm{op}} 
			&\leq \norm{\bm{Q}_1}_{\mathrm{op}}^2
			= \sup_{\bm{u}\in\mathbb{S}^{n-1}}\bm{u}^\prime \bm{Q}_1^\prime \bm{Q}_1\bm{u}\\
			&= \frac{1}{n}\sup_{\bm{u}\in\mathbb{S}^{n-1}}\bm{u}^\prime[\bm{I}_n\otimes\bm{w}_G^\prime]\bm{\Sigma}[\bm{I}_n\otimes\bm{w}_G]\bm{u}\\
			&\leq \frac{1}{n}
			\left(\sup_{\bm{u}\in\mathbb{S}^{n-1}}\frac{\bm{u}^\prime[\bm{I}_n\otimes\bm{w}_G^\prime]\bm{\Sigma}[\bm{I}_n\otimes\bm{w}_G]\bm{u}}{\bm{u}^\prime[\bm{I}_n\otimes\bm{w}_G^\prime][\bm{I}_n\otimes\bm{w}_G]\bm{u}}\right)\\
			&\hspace{5mm}\cdot\left(\sup_{\bm{u}\in\mathbb{S}^{n-1}}\bm{u}^\prime[\bm{I}_n\otimes\bm{w}_G^\prime\bm{w}_G]\bm{u}\right)\\
			&\leq \frac{C_x}{n}\norm{\bm{w}_G}_2^2
			\leq \frac{C_x}{n}\norm{\bm{U}_G}_\mathrm{op}^2\norm{\bm{\Delta}}_2^2
			\leq \frac{\kappa_U}{n}.
		\end{split}
	\end{equation}
	Finally we can use (\ref{eq: trQ}) and (\ref{eq: Q-op}) to bound the Frobenius norm of $\bm{Q}$. 
	By some algebra, for any square matrices $\bm{A}, \bm{B}\in\mathbb{R}^{n\times n}$, $\norm{\bm{AB}}_{\mathrm{F}}^2\leq \norm{\bm{A}}_\mathrm{op}^2\norm{\bm{B}}_\mathrm{F}^2$ holds. 
	Hence,
	\begin{equation*}
		\begin{split}
			\norm{\bm{Q}}_{\mathrm{F}}^2 
			= \norm{\bm{Q}_1\bm{Q}_1^\prime}_{\mathrm{F}}^2 
			\leq \norm{\bm{Q}_1}_\mathrm{op}^2\norm{\bm{Q}_1}_\mathrm{F}^2
			= \norm{\bm{Q}_1}_\mathrm{op}^2\trace({\bm{Q}})
			\leq\frac{\kappa_U^2}{n}.
		\end{split}
	\end{equation*}
	This, together with (\ref{eq:quadraticform}), (\ref{eq: Q-op}) and the Hanson-Wright inequality \citet[Chapter 6]{Vershynin2018}, leads to
	\begin{equation}
		\begin{split}\label{eq:lem3-HW}
			\mathbb{P}\left\{\left|\norm{\bm{\Delta}}_n^2 -\mathbb{E}\norm{\bm{\Delta}}_n^2\right|\geq t\right\}
			&\leq 2\exp\left\{-c_{\mathrm{H}}\min\left(\frac{t}{\norm{ \bm{Q}}}_{\mathrm{op}},\frac{t^2}{\norm{ \bm{Q}}_{\mathrm{F}}^2}\right)\right\}\\
			&\leq 2\exp\left\{-c_{\mathrm{H}}n\min\left({t}/{\kappa_U},({t}/{\kappa_U})^2\right)\right\},
		\end{split}
	\end{equation}
	where $c_{\mathrm{H}}$ is a positive constant. 
	
	On the other hand,
	\[
	\norm{\bm{\Delta}}_n^2 -\mathbb{E}\norm{\bm{\Delta}}_n^2 = \frac{1}{n}\sum_{i=1}^{n}\bm{\Delta}^\prime\bm{z}^i{\bm{z}^{i\prime}} \bm{\Delta} -\mathbb{E}(\bm{\Delta}^\prime\bm{z}^i{\bm{z}^{i\prime}}\bm{\Delta}) = \bm{\Delta}^{\prime}\widehat{\bm{\Gamma} } \bm{\Delta},
	\]
	where $\widehat{\bm{\Gamma} }=n^{-1}\sum_{i=1}^{n}[\bm{z}^i{\bm{z}^{i\prime}} -E(\bm{z}^i{\bm{z}^{i\prime}})]$ is a symmetric matrix.
	Consider a $\varepsilon$-net $\bar{\mathcal{S}}_{2K}$, with the cardinality of $\mathcal{N}(2K,\varepsilon)$, for the set $\mathcal{S}_{2K}$. For any $\bm{\Delta}\in\mathcal{S}_{2K}$, there exists a $\bar{\bm{\Delta}}_j\in\bar{\mathcal{S}}_{2K}$ such that $\|\bm{\Delta}-\bar{\bm{\Delta}}_j\|_{\mathrm{F}}\leq \varepsilon$. 
	Note that $\bm{\Delta}-\bar{\bm{\Delta}}_j\in\widehat{\mathcal{S}}_{4K}$ and, from Lemma \ref{lemma:partition}, we can further find  $\bm{\Delta}_1,\bm{\Delta}_2\in\widehat{\mathcal{S}}_{2K}$ such that $\langle\bm{\Delta}_1,\bm{\Delta}_2\rangle=0$ and $\bm{\Delta}-\bar{\bm{\Delta}}_j =\bm{\Delta}_1+\bm{\Delta}_2$, and it then holds that
	$\|\bm{\Delta}_1\|_{\mathrm{F}}+\|\bm{\Delta}_2\|_{\mathrm{F}}\leq \sqrt{2}\|\bm{\Delta}-\bar{\bm{\Delta}}_j\|_{\mathrm{F}}\leq \sqrt{2}\varepsilon$.
	Moreover, for a general real symmetric matrix $\bm{A}\in \mathbb{R}^{d\times d}$, $\bm{u}\in\mathbb{R}^d$ and $\bm{v}\in\mathbb{R}^d$, $\sup_{\|\bm{u}\|_2=\|\bm{v}\|_2=1}\bm{u}^{\prime}\bm{A}\bm{v}=\sup_{\|\bm{u}\|_2=1}\bm{u}^{\prime}\bm{A}\bm{u}$. As a result,
	\begin{align*}
		\bm{\Delta}^{\prime}\widehat{\bm{\Gamma} } \bm{\Delta} &=\bar{\bm{\Delta}}_j^{\prime}\widehat{\bm{\Gamma}}\bar{\bm{\Delta}}_j + (\bm{\Delta}_1+\bm{\Delta}_2)^{\prime}\widehat{\bm{\Gamma}}(\bm{\Delta}_1+\bm{\Delta}_2+2\bar{\bm{\Delta}}_j)\\
		&\leq \max_{1\leq j\leq \mathcal{N}(2K,\varepsilon)}\bar{\bm{\Delta}}_j^{\prime}\widehat{\bm{\Gamma}}\bar{\bm{\Delta}}_j + 5\varepsilon \sup_{\bm{\Delta}\in \mathcal{S}_{2K}}\bm{\Delta}^{\prime}\widehat{\bm{\Gamma} } \bm{\Delta},
	\end{align*}
	where $(\|\bm{\Delta}_1\|_{\mathrm{F}}+\|\bm{\Delta}_2\|_{\mathrm{F}})^2+2(\|\bm{\Delta}_1\|_{\mathrm{F}}+\|\bm{\Delta}_2\|_{\mathrm{F}})\|\bar{\bm{\Delta}}_j\|_{\mathrm{F}}\leq 2(\varepsilon+\sqrt{2})\varepsilon\leq 5\varepsilon$
	as $\varepsilon \leq 1$, and this leads to
	\begin{equation}\label{eq:lem3-sup}
		\sup_{\bm{\Delta}\in \mathcal{S}_{2K}}\bm{\Delta}^{\prime}\widehat{\bm{\Gamma} } \bm{\Delta} \leq (1-5\varepsilon)^{-1} \max_{1\leq j\leq \mathcal{N}(2K,\varepsilon)}\bar{\bm{\Delta}}_j^{\prime}\widehat{\bm{\Gamma}}\bar{\bm{\Delta}}_j.
	\end{equation}
	Note that, from Lemma \ref{lemma:partition}(b), $\log\mathcal{N}(2K,\varepsilon)\leq 2d_{\mathcal{M}}\log(9/\varepsilon)$, where $d_{\mathcal{M}}=K(L+P+1)$. Let $\varepsilon=1/10$, and then $2\log(9/\varepsilon)<9$. Combining (\ref{eq:lem3-HW}) and (\ref{eq:lem3-sup}) and letting $t = \kappa_L/2$, we have
	\begin{align*}
		&\mathbb{P}\left\{\sup_{\bm{\Delta}\in\mathcal{S}_{2K}}
		\norm{\bm{\Delta}}_n^2 -\mathbb{E}\norm{\bm{\Delta}}_n^2
		\geq\frac{\kappa_L}{2}\right\}\leq2\exp\left\{-c_{\mathrm{H}}n\left(\frac{\kappa_L}{4\kappa_U}\right)^2+9d_{\mathcal{M}}\right\},
	\end{align*}
	which, together with (\ref{eq:mean1}) and the fact that $\kappa_L\leq \kappa_U$, implies that
	\begin{equation}
		\label{eq:lem3-bound}
		\begin{split}
			&\mathbb{P}\left\{\sup_{\bm{\Delta}\in\mathcal{S}_{2K}}
			\norm{\bm{\Delta}}_n^2 
			\geq\alpha_{\mathrm{RSM}}\right\}\leq2\exp\left\{-c_{\mathrm{H}}n\left(\frac{\kappa_L}{4\kappa_U}\right)^2+9d_{\mathcal{M}}\right\},
		\end{split}
	\end{equation}
	where $\alpha_{\mathrm{RSM}}=1.5\kappa_U\geq \kappa_U+\kappa_L/2$.
	
	By a method similar to (\ref{eq:lem3-sup}), we can also show that
	\begin{align*}
		&\sup_{\bm{\Delta}\in \mathcal{S}_{2K}} \mathbb{E}\norm{\bm{\Delta}}_n^2 -\norm{\bm{\Delta}}_n^2\leq (1-5\varepsilon)^{-1} \max_{1\leq j\leq \mathcal{N}(2K,\varepsilon)} \mathbb{E}\norm{\bar{\bm{\Delta}}_j}_n^2 -\norm{\bar{\bm{\Delta}}_j}_n^2,
	\end{align*}
	which, together with (\ref{eq:lem3-HW}), leads to
	\begin{align*}
		&\mathbb{P}\left\{\sup_{\bm{\Delta}\in\mathcal{S}_{2K}}
		\mathbb{E}\norm{\bm{\Delta}}_n^2 -\norm{\bm{\Delta}}_n^2 
		\geq\frac{\kappa_L}{2}\right\}\leq 2\exp\left\{-c_{\mathrm{H}}n\left(\frac{\kappa_L}{4\kappa_U}\right)^2+9d_{\mathcal{M}}\right\},
	\end{align*}
	where $t = \kappa_L/2$ and $\varepsilon=1/10$.
	As a result,
	\begin{equation*}
		\mathbb{P}\left\{\inf_{\bm{\Delta}\in\mathcal{S}_{2K}}
		\norm{\bm{\Delta}}_n^2 
		\leq \alpha_{\mathrm{RSC}}\right\}
		\leq 2\exp\left\{-c_{\mathrm{H}}n\left(\frac{\kappa_L}{4\kappa_U}\right)^2+9d_{\mathcal{M}}\right\},
	\end{equation*}
	where $\alpha_{\mathrm{RSC}}=\kappa_L-\kappa_L/2$.
	This, together with (\ref{eq:lem3-bound}), accomplishes the proof.
\end{proof}

\begin{lemma}(Concentration bound for martingale)\label{lemma:concentration}
	Let $\{\xi^i,1\leq i\leq n\}$ be independent $\sigma^2$-sub-Gaussian random variables with mean zero, and $\{z^i,1\leq i\leq n\}$ is another sequence of random variables. Suppose that $\xi^i$ is independent of $\{z^i,z^{i-1},\ldots,z^1\}$ for all $1\leq i\leq n$. It then holds that, for any real numbers $\alpha, \beta>0$,
	$$\mathbb{P}\left[\left\{\frac{1}{n}\sum_{i=1}^{n}\xi^iz^i\geq \alpha\right\}\cap\left\{\frac{1}{n}\sum_{i=1}^{n}(z^{i})^2\leq\beta\right\}\right]\leq \exp\left(-\frac{n\alpha^2}{2\sigma^2\beta}\right).$$
\end{lemma}
\begin{proof}
	We can prove the lemma by a method similar to Lemma 4.2 in \cite{simchowitz2018learning}.
\end{proof}

\newpage

\renewcommand{\thesection}{Appendix D}
\renewcommand{\thelemma}{D.\arabic{lemma}}
\renewcommand{\thetheorem}{D.\arabic{theorem}}
\renewcommand{\thecorollary}{D.\arabic{corollary}}
\renewcommand{\theassumption}{D.\arabic{assumption}}
\renewcommand{\thesubsection}{D.\arabic{subsection}}
\renewcommand{\theequation}{D.\arabic{equation}}
\setcounter{lemma}{0}
\setcounter{equation}{0}

\section{Classification Problems} \label{app:classification}

Starting from Section 2.2, we know that, for each input tensor $\cm{X}$, the intermediate scalar output after convolution and pooling has the following form
\begin{align*}
	\text{output} &= \langle\cm{X},\sum_{k=1}^{K}(\cm{B}_k\otimes\cm{A}_k)\times_1\bm{U}^{(1)}_{\mathcal{F}}\times_2\bm{U}^{(2)}_{\mathcal{F}}\times_3 \cdots \times_N\bm{U}^{(N)}_{\mathcal{F}}\rangle = \langle\cm{Z}, \cm{W}\rangle,
\end{align*}
where $\cm{W}=\sum_{k=1}^{K}\cm{B}_k\otimes\cm{A}_k$ and ${\cm{Z}} = \cm{X}\times_1\bm{U}^{(1)\prime}_\mathcal{F}\times_2\bm{U}^{(2)\prime}_\mathcal{F}\times_3\cdots\times_N\bm{U}^{(N)\prime}_\mathcal{F}$, $\cm{A}_k\in\mathbb{R}^{l_1\times l_2\times\cdots\times l_N}$ is the $k$th kernel and 
$\cm{B}_k\in\mathbb{R}^{p_1\times p_2\times\cdots\times p_N}$ is the corresponding $k$th fully-connected weight tensor.

Consider a binary classification problem. We have the binary label output $y\in\{0,1\}$ with
\begin{align*}
	&p(y|\cm{Z}) = \left(\frac{1}{1+\exp(\langle\cm{Z},\cm{W}\rangle)}\right)^{1-y}\left(\frac{\exp(\langle\cm{Z},\cm{W}\rangle)}{1+\exp(\langle\cm{Z},\cm{W}\rangle)}\right)^{y} \\
	&\hspace{5mm}= \exp\left\{y\langle\cm{Z},\cm{W}\rangle - \log\left[1+\exp(\langle\cm{Z},\cm{W}\rangle)\right]\right\}.
\end{align*}
Suppose we have samples $\{\cm{Z}^i,y^i\}_{i=1}^n$, we use the negative log-likelihood function as our loss function. It is given, up to a scaling of $n^{-1}$ by
\begin{equation} \label{eq:loss}
	\mathcal{L}_n(\cm{W}) =-\frac{1}{n}\sum_{i=1}^{n}y^i\langle\cm{Z}^i,\cm{W}\rangle + \frac{1}{n}\sum_{i=1}^{n}\phi(\langle\cm{Z}^i,\cm{W}\rangle),
\end{equation}
where $\phi(z) = \log(1+e^z)$ and its gradient and Hessian matrix is 
\begin{equation}
	\begin{split}
		&\nabla\mathcal{L}_n(\cm{W}) = \frac{\partial\mathcal{L}_n(\cm{W})}{\partial\vectorize(\cm{W})}= -\frac{1}{n}\sum_{i=1}^{n}y^i\vectorize(\cm{Z}^i)+\frac{1}{n}\sum_{i=1}^{n}\phi^\prime(\langle\cm{Z}^i,\cm{W}\rangle)\vectorize(\cm{Z}^i) 
	\end{split}\label{eq:grad}
\end{equation}
\begin{equation}
	\begin{split}
		&\bm{H}_n(\cm{W}) = \frac{\partial^2\mathcal{L}_n(\cm{W})}{\partial^2\vectorize(\cm{W})} = 
		\frac{1}{n}\sum_{i=1}^{n}\phi^{\prime\prime}(\langle\cm{Z}^i,\cm{W}\rangle)\vectorize(\cm{Z}^i)\vectorize(\cm{Z}^i)^\prime, 
	\end{split}\label{eq:hessian}
\end{equation}
with $\phi^\prime(z) = 1/(1+e^{-z})\in(0,1)$ and $\phi^{\prime\prime}(z) = e^z/(1+e^z)^2 = 1/(e^{-z} + 2 + e^z)\in(0,0.25)$ [because $e^{-z}+ e^z\geq 2$]. Since $\bm{H}_n(\cm{W})$ is a positive semi-definite matrix, our loss function in (\ref{eq:loss}) is convex. 

Suppose $\cm{\widehat{W}}$ is the minimizer to the loss function
\[
\cm{\widehat{W}} \in \argmin_{\cm{W}\in\widehat{\mathcal{S}}_K\cap\mathbb{B}(R)}\mathcal{L}_n(\cm{W}),
\]
where $\widehat{\mathcal{S}}_K=\{\sum_{k=1}^{K}\cm{B}_k\otimes\cm{A}_k: \cm{A}_k\in\mathbb{R}^{l_1\times l_2\times\cdots\times l_N} \text{ and } \cm{B}_k\in\mathbb{R}^{p_1\times p_2\times\cdots\times p_N} \}$, and $\mathbb{B}(R)$ is a Frobenius ball of some fixed radius $R$ centered at the underlying true parameter.

Similar to \citet{fan2019generalized}, we need to make two additional assumptions to guarantee the locally strong convexity condition.

\begin{assumption}[Classification] \label{assum}
	Apart from Assumption 1(ii)\&(iii), we additionally assume that
	
	(C1) $\{\vectorize(\cm{X}^i)\}_{i=1}^n$ are i.i.d gaussian vectors with mean zero and variance $\bm{\Sigma}$, where $\bm{\Sigma}\leq C_x\bm{I}$.
	
	(C2) the Hessian matrix at the underlying true parameter $\cm{W}^*$ is positive definite and there exists some $\kappa_0>\kappa_U>0$, such that $\bm{H}(\cm{W}^*) = \mathbb{E}(\bm{H}_n(\cm{W}^*))\geq \kappa_0\cdot\bm{I}$, where $\kappa_U=C_xC_u$;
	
	(C3) $\|\cm{W}^*\|_\mathrm{F} \geq \alpha\sqrt{d_{\mathcal{M}}}$ for some constant $\alpha$, where $d_{\mathcal{M}} = K(P+L+1)$.
\end{assumption}

Notice that we can relax (C1) into Assumption 1(i), but it will require more techincal details. Also, $\{\vectorize(\cm{X}^i)\}_{i=1}^n$ can be sub-gaussian random vectors instead of gaussian random vectors.

Denote $d_{\mathcal{M}} = K(P+L+1)$.
\begin{theorem}[Classification: CNN] \label{thm:classification}
	Suppose that Assumption 1(ii)\&(iii) and Assumption \ref{assum} hold and $n\gtrsim d_{\mathcal{M}}$. Then,
	\[
	\|\cm{\widehat{W}} - \cm{W}^*\|_\mathrm{F} \lesssim \frac{2\sqrt{\kappa_U}}{\kappa_1}\sqrt{\frac{d_{\mathcal{M}}}{n}},
	\]
	with probability $1-4\exp\left\{-0.25cn + 9d_{\mathcal{M}}\right\}-2\exp\left\{-c_\gamma d_{\mathcal{M}}\right\}$ where $c$ and $c_\gamma$ are some positive constants.
\end{theorem}
Denote $d_{\mathcal{M}}^\mathrm{TU} = r\prod_{j=1}^{N}R_j+\sum_{i=1}^{N}l_iR_i+RP$ and $d_{\mathcal{M}}^\mathrm{CP} = R^{N+1}+R(\sum_{i=1}^{N}l_i+P)$.
\begin{corollary}[Classification: Compressed CNN]
	Let $(\cm{\widehat{W}},d_{\mathcal{M}})$ be $(\cm{\widehat{W}}_{\mathrm{TU}},d_{\mathcal{M}}^\mathrm{TU})$ for Tucker decomposition, or $(\cm{\widehat{W}}_{\mathrm{CP}},d_{\mathcal{M}}^\mathrm{CP})$ for CP decomposition.
	Suppose that Assumptions in Theorem \ref{thm:classification} hold and $n\gtrsim c_Nd_{\mathcal{M}}$. Then,
	\[
	\|\cm{\widehat{W}} - \cm{W}^*\|_\mathrm{F} \lesssim \frac{2\sqrt{\kappa_U}}{\kappa_1}\sqrt{\frac{3c_Nd_{\mathcal{M}}}{n}},
	\]
	with probability $1-4\exp\left\{-0.25cn + 3c_Nd_{\mathcal{M}}\right\}-2\exp\left\{-c_\gamma d_{\mathcal{M}}\right\}$ where $c$ and $c_\gamma$ are some positive constants, and $c_N$ is defined as in Theorem 2.
\end{corollary}

We consider a binary classification problem as a simple illustration. In fact, the analysis framework can be easily extended to a multiclass classification problem. Here, we consider a $M$-class classification problem.

Because we need our intermediate output after convolution and pooling to be a vector of length $M$, instead of a scalar, we need to introduce one additional dimension to the fully-connected weight tensor. 
Hence, we introduce another subscript $m$ to represent the class label. And the set of fully-connected weights is represented as $\{\cm{B}_{k,m}\}_{1\leq k\leq K,1\leq m\leq M}$ where each $\cm{B}_{k,m}$ is of size $p_1\times p_2\times\cdots\times p_N$. 

Then, for each input tensor $\cm{X}$, our intermediate output is a vector of size $M$, where the $m$th entry is represented by
\[
\text{output}_{m} = \langle\cm{Z}, \cm{W}_m\rangle,
\]
where $\cm{W}_m = \sum_{k=1}^{K}(\cm{B}_{k,m}\otimes\cm{A}_k)$ is an $N$-th order tensor of size $l_1p_1\times l_2p_2\times\cdots\times l_Np_N$.

For a $M$-class classification problem, we have the vector label output $\bm{y}\in(0,1)^M$. Essentially, each entry of $\bm{y}$ comes from a different binary classification problem, with $M$ problems in total. We can model it as
\begin{align*}
	p(y_m|\cm{Z}) &= \left(\frac{1}{1+\exp(\langle\cm{Z},\cm{W}_m\rangle)}\right)^{1-y_m}\left(\frac{\exp(\langle\cm{Z},\cm{W}_m\rangle)}{1+\exp(\langle\cm{Z},\cm{W}_m\rangle)}\right)^{y_m}\\
	&=
	\exp\left\{y_m\langle\cm{Z},\cm{W}_m\rangle - \log\left[1+\exp(\langle\cm{Z},\cm{W}_m\rangle)\right]\right\}
\end{align*}

If we stack $\{\cm{W}_m\}_{m=1}^M$ into a tensor $\cm{W}_\mathrm{stack}$, which is an $N+1$-order tensor of size $l_1p_1\times l_2p_2\times\cdots\times l_Np_N\times M$. And we further introduce some natural basis vectors $\{\bm{e}_m\in\mathbb{R}^M\}_{m=1}^M$. It can be shown that
\[
\langle\cm{Z}^i,\cm{W}_m\rangle = \langle\underbrace{\cm{Z}^i\circ\bm{e}_m}_{\cm{Z}^i_m},\cm{W}_\mathrm{stack}\rangle,
\]
where $\circ$ is the outer product.

We can then recast this model into one with $nM$ samples $\{\cm{Z}^i_m, y^i_m:1\leq k\leq K,1\leq m\leq M\}$. The corresponding loss function is
\begin{align*}
	&\mathcal{L}_n(\cm{W}_\mathrm{stack}) =-\frac{1}{nM}\sum_{m=1}^{M}\sum_{i=1}^{n}y^i\langle\cm{Z}^i_m,\cm{W}_\mathrm{stack}\rangle\\
	&\hspace{5mm} + \frac{1}{nM}\sum_{m=1}^{M}\sum_{i=1}^{n}\phi(\langle\cm{Z}^i_m,\cm{W}_\mathrm{stack}\rangle).
\end{align*}

We can now use the techniques in Theorem \ref{thm:classification} to show the following corollaries for multiclass classification problem.

Denote $d_{\mathcal{M}}^\mathrm{MC} = K(MP+L+1)$. 
\begin{corollary}[Multiclass Classification: CNN]
	Under similar assumptions as in Theorem \ref{thm:classification}, suppose that $n\gtrsim d_{\mathcal{M}}^\mathrm{MC}$. Then,
	\[
	\|\cm{\widehat{W}}_\mathrm{stack} - \cm{W}_\mathrm{stack}^*\|_\mathrm{F} \lesssim \frac{2\sqrt{\kappa_U}}{\kappa_1}\sqrt{\frac{d_{\mathcal{M}}^\mathrm{MC}}{n}},
	\]
	with probability $1-4\exp\left\{-0.25cn + 9d_{\mathcal{M}}^\mathrm{MC}\right\}-2\exp\left\{-c_\gamma d_{\mathcal{M}}^\mathrm{MC}\right\}$, where $c$ and $c_\gamma$ are some positive constants.
\end{corollary}

Denote $d_{\mathcal{M}}^\mathrm{MC-TU} = R\prod_{j=1}^{N}R_j+\sum_{i=1}^{N}l_iR_i+RPM$ and $d_{\mathcal{M}}^\mathrm{MC-CP} = R^{N+1}+R(\sum_{i=1}^{N}l_i+PM)$.
\begin{corollary}[Multiclass Classification: Compressed CNN]
	Let $(\cm{\widehat{W}_\mathrm{stack}},d_{\mathcal{M}}^\mathrm{MC})$ be $(\cm{\widehat{W}}_{\mathrm{stack,TU}},d_{\mathcal{M}}^\mathrm{MC-TU})$ for Tucker decomposition, or $(\cm{\widehat{W}}_{\mathrm{stack,CP}},d_{\mathcal{M}}^\mathrm{MC-CP})$ for CP decomposition.
	Suppose that Assumptions in Theorem \ref{thm:classification} hold and $n\gtrsim c_Nd_{\mathcal{M}}^\mathrm{multi}$. Then,
	\[
	\|\cm{\widehat{W}}_\mathrm{stack} - \cm{W}_\mathrm{stack}^*\|_\mathrm{F} \lesssim \frac{2\sqrt{\kappa_U}}{\kappa_1}\sqrt{\frac{3c_Nd_{\mathcal{M}}^\mathrm{MC}}{n}},
	\]
	with probability $1-4\exp\left\{-0.25cn + 3c_Nd_{\mathcal{M}}^\mathrm{MC}\right\}-2\exp\left\{-c_\gamma d_{\mathcal{M}}^\mathrm{MC}\right\}$, where $c$ and $c_\gamma$ are some positive constants, and $c_N$ is defined as in Theorem 2.
\end{corollary}

\begin{proof}[Proof of Theorem \ref{thm:classification}]
	Denote the sets $\widehat{\mathcal{S}}_K=\{\sum_{k=1}^{K}\cm{B}_k\otimes\cm{A}_k: \cm{A}_k\in\mathbb{R}^{l_1\times l_2\times\cdots\times l_N} \text{ and } \cm{B}_k\in\mathbb{R}^{p_1\times p_2\times\cdots\times p_N} \}$ and $\mathcal{S}_K=\{\cm{W}\in \widehat{\mathcal{S}}_K: \|\cm{W}\|_{\mathrm{F}}=1\}$. 
	We further denote $\bm{\Delta} = \vectorize(\cm{W}-\cm{W}^*)$, where $\cm{W}^*$ is the underlying true parameter and $\cm{W},\cm{W}^*\in\widehat{\mathcal{S}}_K$, and define the first-order Taylor error
	\[
	\cm{E}_n(\bm{\Delta}) = \mathcal{L}_n(\cm{W}) - \mathcal{L}_n(\cm{W}^*) -\left\langle\nabla\mathcal{L}_n(\cm{W}^*),\bm{\Delta}\right\rangle.
	\]
	Suppose $\cm{\widehat{W}}$ is the minimizer for the loss function, i.e.,
	\[
	\cm{\widehat{W}} = \argmin_{\cm{W}\in\widehat{\mathcal{S}}_K}\mathcal{L}_n(\cm{W}).
	\]
	Denote $\widehat{\bm{\Delta}} = \vectorize(\cm{\widehat{W}}-\cm{W}^*)$.
	We then have
	\[
	\mathcal{L}_n(\cm{\widehat{W}}) - \mathcal{L}_n(\cm{W}^*) \leq 0,
	\]
	which can be rearranged into
	\[
	\cm{E}_n(\widehat{\bm{\Delta}}) \leq -\left\langle\nabla\mathcal{L}_n(\cm{W}^*),\widehat{\bm{\Delta}}\right\rangle.
	\]
	Then, for some $\cm{\widetilde{W}}$ between $\cm{\widehat{W}}$ and $\cm{W}^*$,
	\[
	\frac{1}{2}\widehat{\bm{\Delta}}^\prime\bm{H}_n(\cm{\widetilde{W}})\widehat{\bm{\Delta}}\leq \left|\left\langle\nabla\mathcal{L}_n(\cm{W}^*),\widehat{\bm{\Delta}}\right\rangle\right|,
	\]
	which leads to 
	\begin{equation}
		\|\widehat{\bm{\Delta}}\|_\mathrm{F}^2 \sup_{\bm{\Delta}\in\mathcal{S}_{2K}}\bm{\Delta}^\prime\bm{H}_n(\cm{\widetilde{W}})\bm{\Delta} \leq 2\|\widehat{\bm{\Delta}}\|_\mathrm{F}\sup_{\bm{\Delta}\in\mathcal{S}_{2K}}\left|\left\langle\nabla\mathcal{L}_n(\cm{W}^*),\bm{\Delta}\right\rangle\right|.
	\end{equation}
	
	From Lemma \ref{lem:dev} and Lemma \ref{lem:LRSC}, when $n\gtrsim d_{\mathcal{M}}$, we obtain that, for some $\delta>0$,
	\[
	\|\cm{\widehat{W}} - \cm{W}^*\|_\mathrm{F} \leq \frac{2\sqrt{\kappa_U}}{\kappa_1}\left(\sqrt{\frac{d_{\mathcal{M}}}{n}}+\sqrt{\frac{\delta}{n}}\right),
	\]
	with probability $1-4\exp\left\{-0.25cn + 9d_{\mathcal{M}}\right\}-2\exp\left\{-c_\gamma d_{\mathcal{M}}-c\delta\right\}$.
	
\end{proof}

Now we proof several lemmas to be used in Theorem \ref{thm:classification}. For simplicity of notation, denote $\bm{x}^i = \vectorize(\cm{X}^i)$ and  $\bm{z}^i = \vectorize(\cm{Z}^i)$. 
It holds that $\bm{z}^i = \bm{U}_G\bm{x}^i$, for $1\leq i\leq n$. 
For a random variable $x$, we denote its sub-gaussian norm as $\|x\|_{\psi_2}:= \sup_{p\geq 1}(\mathbb{E}(|x|^p)^{1/p})/\sqrt{p}$ and its sub-exponential norm as $\|x\|_{\psi_1}:= \sup_{p\geq 1}(\mathbb{E}(|x|^p)^{1/p})/p$.

\begin{lemma}[Deviation bound] \label{lem:dev}
	Under Assumption \ref{assum}(C3), suppose that $n\gtrsim d_{\mathcal{M}}$, then
	\[
	\mathbb{P}\left\{\sup_{\bm{\Delta}\in\mathcal{S}_{2K}}\left|\langle\nabla\mathcal{L}_n(\cm{W}^*),\bm{\Delta}\rangle\right|\geq 0.5\sqrt{\kappa_U}\sqrt{\frac{d_{\mathcal{M}}}{n}}\right\}\leq 2\exp\left\{-c_\gamma d_{\mathcal{M}}\right\}, 
	\]
	$d_{\mathcal{M}} = K(P+L+1)$, $\kappa_U = C_xC_u$ and $c_\gamma$ is some positive constant.
\end{lemma}
\begin{proof}
	Let $\eta^i=\langle\cm{Z}^i,\cm{W}^*\rangle$, and from (\ref{eq:grad}),
	\[
	\langle\nabla\mathcal{L}_n(\cm{W}^*),\bm{\Delta}\rangle = \frac{1}{n}\sum_{i=1}^{n}[\phi^\prime(\eta^i) - y^i]\langle\bm{z}^i,\bm{\Delta}\rangle,
	\]
	and we can observe that 
	\begin{align*}
		&\mathbb{E}\{[\phi^\prime(\eta^i) - y^i]\langle\bm{z}^i,\bm{\Delta}\rangle\} = \mathbb{E}\{\mathbb{E}[\phi^\prime(\eta^i) - y^i|\bm{z}^i]\langle\bm{z}^i,\bm{\Delta}\rangle\} =
		0\\
		&\hspace{20mm}=
		\mathbb{E}[\phi^\prime(\eta^i) - y^i]\mathbb{E}[\langle\bm{z}^i,\bm{\Delta}\rangle].	
	\end{align*}
	It implies (i) $\mathbb{E}\langle\nabla\mathcal{L}_n(\cm{W}^*),\bm{\Delta}\rangle = 0$, (ii) the independence between $\phi^\prime(\eta^i) - y^i$ and $\langle\bm{z}^i,\bm{\Delta}\rangle$. And from Lemma \ref{lem:sub-exp}, the independence leads to $\|[\phi^\prime(\eta^i) - y^i]\langle\bm{z}^i,\bm{\Delta}\rangle\|_{\psi_1}\leq\|\phi^\prime(\eta^i) - y^i\|_{\psi_2}\|\langle\bm{z}^i,\bm{\Delta}\rangle\|_{\psi_2}$. 
	Denote $\kappa_U = C_xC_u$. For any fixed $\bm{\Delta}$ such that $\|\bm{\Delta}\|_2 = 1$,
	\[
	\|\langle\bm{z}^i,\bm{\Delta}\rangle\|_{\psi_2} = \|\langle\bm{w}^i,\bm{\Sigma}^{1/2}\bm{U}_G\bm{\Delta}\rangle\|_{\psi_2}\leq \sqrt{\kappa_U}.
	\]
	This, together with $\|\phi^\prime(\eta^i) - y^i\|_{\psi_2}\leq 0.25$ in Lemma \ref{lem:sug-gaussian}, gives us
	\[
	\|[\phi^\prime(\eta^i) - y^i]\langle\bm{z}^i,\bm{\Delta}\rangle\|_{\psi_1} \leq 0.25\sqrt{\kappa_U}.
	\]
	Then, we can use the Beinstein-type inequality, namely Corollary 5.17 in \citet{vershynin2010introduction} to derive that, for any fixed $\bm{\Delta}$ with unit $l_2$-norm,
	\begin{equation} \label{eq:dev}
		\begin{split}
			\mathbb{P}\left\{\left|\langle\nabla\mathcal{L}_n(\cm{W}^*),\bm{\Delta}\rangle\right|\geq t\right\} &=
			\mathbb{P}\left\{\frac{1}{n}\left|\sum_{i=1}^{n}[\phi^\prime(\eta^i) - y^i]\langle\bm{z}^i,\bm{\Delta}\rangle\right|\geq t\right\}
			\\
			&\leq		2\exp\left\{-cn\min\left(\frac{4t}{\sqrt{\kappa_U}},\frac{16t^2}{\kappa_U}\right)\right\}.
		\end{split}	
	\end{equation}
	
	Consider a $\varepsilon$-net $\bar{\mathcal{S}}_{2K}$, with the cardinality of $\mathcal{N}(2K,\varepsilon)$, for the set $\mathcal{S}_{2K}$. For any $\bm{\Delta}\in\mathcal{S}_{2K}$, there exists a $\bar{\bm{\Delta}}_j\in\bar{\mathcal{S}}_{2K}$ such that $\|\bm{\Delta}-\bar{\bm{\Delta}}_j\|_{\mathrm{F}}\leq \varepsilon$.
	Note that $\bm{\Delta}-\bar{\bm{\Delta}}_j\in\widehat{\mathcal{S}}_{4K}$ and, from Lemma \ref{lemma:partition}(a), we can further find  $\bm{\Delta}_1,\bm{\Delta}_2\in\widehat{\mathcal{S}}_{2K}$ such that $\langle\bm{\Delta}_1,\bm{\Delta}_2\rangle=0$ and $\bm{\Delta}-\bar{\bm{\Delta}}_j =\bm{\Delta}_1+\bm{\Delta}_2$. It then holds that
	$\|\bm{\Delta}_1\|_{\mathrm{F}}+\|\bm{\Delta}_2\|_{\mathrm{F}}\leq \sqrt{2}\|\bm{\Delta}-\bar{\bm{\Delta}}_j\|_{\mathrm{F}}\leq \sqrt{2}\varepsilon$ since $\|\bm{\Delta}-\bar{\bm{\Delta}}_j\|_{\mathrm{F}}^2=\|\bm{\Delta}_1\|_{\mathrm{F}}^2+\|\bm{\Delta}_2\|_{\mathrm{F}}^2$.
	As a result,
	\begin{align*}
		\left|\langle\nabla\mathcal{L}_n(\cm{W}^*),\bm{\Delta}\rangle\right|
		& =\left|\langle\nabla\mathcal{L}_n(\cm{W}^*),\bar{\bm{\Delta}}_j\rangle\right|
		+\left|\langle\nabla\mathcal{L}_n(\cm{W}^*),\bm{\Delta}_1\rangle\right|+\left|\langle\nabla\mathcal{L}_n(\cm{W}^*),\bm{\Delta}_2\rangle\right|\\
		&\leq \max_{1\leq j\leq \mathcal{N}(2K,\varepsilon)}\left|\langle\nabla\mathcal{L}_n(\cm{W}^*),\bar{\bm{\Delta}}_j\rangle\right| +\sqrt{2}\varepsilon \sup_{\bm{\Delta}\in \mathcal{S}_{2K}}\left|\langle\nabla\mathcal{L}_n(\cm{W}^*),\bm{\Delta}\rangle\right|,
	\end{align*}
	which leads to
	\begin{align*}
		&\sup_{\bm{\Delta}\in \mathcal{S}_{2K}}\left|\langle\nabla\mathcal{L}_n(\cm{W}^*),\bm{\Delta}\rangle\right| \leq (1-\sqrt{2}\varepsilon)^{-1}\max_{1\leq j\leq \mathcal{N}(2K,\varepsilon)}\left|\langle\nabla\mathcal{L}_n(\cm{W}^*),\bar{\bm{\Delta}}_j\rangle\right|.
	\end{align*}
	
	Note that, from Lemma \ref{lemma:partition}(b), $\log\mathcal{N}(2K,\varepsilon)\leq 2d_{\mathcal{M}}\log(9/\varepsilon)$, where $d_{\mathcal{M}}=K(P+L+1)$.
	Let $\varepsilon=(2\sqrt{2})^{-1}$ and then $2\log(9/\varepsilon)<7$. With (\ref{eq:dev}), we can show that
	\begin{align*}
		\mathbb{P}\left\{\sup_{\bm{\Delta}\in\mathcal{S}_{2K}}\left|\langle\nabla\mathcal{L}_n(\cm{W}^*),\bm{\Delta}\rangle\right|\geq 2t\right\}\leq 2\exp\left\{-cn\min\left(\frac{4t}{\sqrt{\kappa_U}},\frac{16t^2}{\kappa_U}\right) + 7d_{\mathcal{M}}\right\}.
	\end{align*}
	Take $t = 0.25\sqrt{\kappa_U}\sqrt{d_{\mathcal{M}}/n}$, and there exists some $\gamma$ such that $\sqrt{d_{\mathcal{M}}/n}\leq\gamma$ holds. We can finally show that
	\begin{align*}
		&\mathbb{P}\left\{\sup_{\bm{\Delta}\in\mathcal{S}_{2K}}\left|\langle\nabla\mathcal{L}_n(\cm{W}^*),\bm{\Delta}\rangle\right|\geq 0.5\sqrt{\kappa_U}\sqrt{\frac{d_{\mathcal{M}}}{n}}\right\}\\
		&\hspace{5mm}\leq 2\exp\left\{-c_\gamma d_{\mathcal{M}}\right\},
	\end{align*}
	where $c_\gamma$ is some positive constant related to $\gamma$.
\end{proof}
\begin{lemma} [LRSC] \label{lem:LRSC}
	Suppose that $n\gtrsim d_{\mathcal{M}}$, under Assumptions \ref{assum}, there exists some constant $R>0$, such that for any $\cm{W}\in\widehat{\mathcal{S}}_{2K}$ satisfying $\|\cm{W} - \cm{W}^*\|_\mathrm{F}\leq R$,
	\[
	\inf_{\bm{\Delta}\in\mathcal{S}_{2K}}\bm{\Delta}^\prime\bm{H}_n(\cm{W})\bm{\Delta}\geq \frac{\widetilde{\kappa}_1}{2}
	\]
	holds with probability 
	\[
	1-4\exp\left\{-0.25cn + 9d_{\mathcal{M}}\right\},
	\]
	where $\widetilde{\kappa}_1 = \kappa_1 - \kappa_U$. And $\kappa_1$ is defined in Lemma \ref{lem:h}.
\end{lemma}
\begin{proof}
	We divide this proof into two parts.
	
	1. RSC of $\mathcal{L}_n(\cm{W})$ at $\cm{W} = \cm{W}^*$.
	
	We first show that, for all $\bm{\Delta}\in\mathcal{S}_{2K}$, the following holds that
	\[
	\bm{\Delta}^\prime\bm{H}_n(\cm{W}^*)\bm{\Delta} \geq \widetilde{\kappa},
	\]
	with probability at least $1-2\exp\left\{-0.25c({\kappa_L}/{\kappa_U})^2n+ 9d_{\mathcal{M}}\right\}$, where $\widetilde{\kappa} = \kappa_0 - \kappa_L>0$.
	
	Let $\eta^i=\langle\cm{Z}^i,\cm{W}^*\rangle$ and denote $\widetilde{\bm{z}}^i = \sqrt{\phi^{\prime\prime}(\eta^i)}\bm{z}^i$ and we can see that
	\[
	\bm{H}_n(\cm{W^*}) = \frac{1}{n}\sum_{i=1}^n\widetilde{\bm{z}}^i\widetilde{\bm{z}}^{i\prime}\text{ and,}\quad \bm{H}(\cm{W}^*) = \mathbb{E}\bm{H}_n(\cm{W^*}).
	\]
	Denote $\bm{x}^i = \vectorize(\cm{X}^i)$. Here, for simplicity, we assume $\{\bm{x}^i\}_{i=1}^n$ to be independent gaussian vectors with mean zero and covariance matrix $\bm{\Sigma}$, where $c_x\bm{I}\leq \bm{\Sigma}\leq C_x\bm{I}$ for some $0<c_x<C_x$. We will also use the notation of $\bm{\Delta}$ for its vectorized version, $\vectorize(\bm{\Delta})$, and we consider $\bm{\Delta}$ with unit $l_2$-norm.
	
	Since $\|\langle\bm{\Delta},\widetilde{\bm{z}}^i\rangle\|_{\psi_2} = \|\langle\bm{\Delta},\sqrt{\phi^{\prime\prime}(\eta^i)}\bm{U}_G^\prime\bm{\Sigma}^{1/2}\bm{w}^i\rangle\|_{\psi_2}\leq 0.25\sqrt{\kappa_U}$, where $\bm{w}^i$ is a standard gaussian vector and $\kappa_U = C_xC_u$, we can show that
	\begin{align*}
		&\|\langle\bm{\Delta},\widetilde{\bm{z}}^i\rangle^2 - \mathbb{E}\left[\langle\bm{\Delta},\widetilde{\bm{z}}\rangle^2\right]\|_{\psi_1} \leq 2\|\langle\bm{\Delta},\widetilde{\bm{z}}^i\rangle^2\|_{\psi_1}\leq 4\|\langle\bm{\Delta},\widetilde{\bm{z}}^i\rangle\|^2_{\psi_2}\leq \kappa_U,
	\end{align*}
	where the first inequality comes from Remark 5.18 in \citet{vershynin2010introduction} and second inequality comes from Lemma 5.14 in \citet{vershynin2010introduction}.
	
	And hence, by the Beinstein-type inequality in Corollary 5.17 in \citet{vershynin2010introduction}, for any fixed $\bm{\Delta}$ such that $\|\bm{\Delta}\|_2 = 1$, we have
	\begin{align*}
		&\hspace{5mm}\mathbb{P}\left\{\left|\bm{\Delta}^\prime(\bm{H}_n(\cm{W}^*)-\bm{H}(\cm{W}^*))\bm{\Delta}\right|\geq t\right\}\\
		&=
		\mathbb{P}\left\{\frac{1}{n}\left|\sum_{i=1}^{n}\left\{\langle\bm{\Delta},\widetilde{\bm{z}}^i\rangle^2 - \mathbb{E}\left[\langle\bm{\Delta},\widetilde{\bm{z}}\rangle^2\right]\right\}\right|\geq t\right\}\\
		&\leq
		2\exp\left\{-cn\min\left(\frac{t}{\kappa_U},\frac{t^2}{\kappa_U^2}\right)\right\}.
	\end{align*}
	
	With similar covering number argument as presented in Lemma 2 in our  paper, we can show that,
	\begin{align*}
		&\mathbb{P}\left\{\sup_{\bm{\Delta}\in\mathcal{S}_{2K}}\left|\bm{\Delta}^\prime(\bm{H}_n(\cm{W}^*)-\bm{H}(\cm{W}^*))\bm{\Delta}\right|\geq 2t\right\}\\
		&\leq 2\exp\left\{-cn\min\left(\frac{t}{\kappa_U},\frac{t^2}{\kappa_U^2}\right)+ 9d_{\mathcal{M}}\right\},
	\end{align*}
	where $d_{\mathcal{M}} = K(P+L+1)$.
	
	Let $t=0.5\kappa_U$. By Assumption \ref{assum}(C2), we can obtain that, when $n\gtrsim d_{\mathcal{M}}$,
	\[	\mathbb{P}\left\{\inf_{\bm{\Delta}\in\mathcal{S}_{2K}}\bm{\Delta}^\prime\bm{H}_n(\cm{W}^*)\bm{\Delta} \leq \widetilde{\kappa}\right\} \leq 2\exp\left\{-0.25cn+ 9d_{\mathcal{M}}\right\},
	\]
	where $\widetilde{\kappa} = \kappa_0 - \kappa_U>0$.
	
	2. LRSC of $\mathcal{L}_n(\cm{W})$ around $\cm{W}^*$.
	
	Define the event $$A = \{|\langle\cm{W}^*,\cm{Z}^i\rangle|>\tau\sup_{\bm{\Delta}\in \mathcal{S}_{2K}}|\langle\bm{\Delta},\bm{z}^i\rangle|\}$$ and construct the functions
	\begin{align*}
		&\bm{h}_n(\cm{W}) = \frac{1}{n}\sum_{i=1}^{n}\phi^{\prime\prime}(\langle\cm{Z}^i,\cm{W}\rangle)\mathbb{I}_A\bm{z}^i\bm{z}^{i\prime},\\
		&\bm{h}(\cm{W}) = \mathbb{E}\bm{h}_n(\cm{W}),
	\end{align*}
	where $\tau$ is some positive constant to be selected according to Lemma \ref{lem:h}.
	Since the difference between $\bm{h}_n(\cdot)$ and $\bm{H}_n(\cdot)$ is the indicator function, it holds that $\bm{H}_n(\cdot)\geq\bm{h}_n(\cdot)$.
	
	We will finish the proof of LRSC in two steps. Firstly, we show that, with high probability, $\bm{h}_n(\cm{W}^*)$ is positive definite on the restricted set $\mathcal{S}_{2K}$. Secondly, we bound the difference between $\bm{\Delta}^\prime\bm{h}_n(\cm{W})\bm{\Delta}$ and $\bm{\Delta}^\prime\bm{h}_n(\cm{W}^*)\bm{\Delta}$, and hence show that $\bm{h}_n(\cm{W})$ is locally positive definite around $\cm{W}^*$. This naturally leads to the LRSC of $\mathcal{L}_n(\cm{W})$ around $\cm{W}^*$.
	
	From Lemma \ref{lem:h}, we can select $\tau$, such that $\bm{h}(\cm{W}^*)\geq \kappa_1\bm{I}$. Following similar arguments as in the first part, we can show that for all $\bm{\Delta}\in\mathcal{S}_{2K}$, the following holds with probability at least $1-2\exp\left\{-0.25cn+ 9d_{\mathcal{M}}\right\}$,
	\begin{equation}\label{eq:h-true}
		\bm{\Delta}^\prime\bm{h}_n(\cm{W}^*)\bm{\Delta} \geq \widetilde{\kappa}_1,
	\end{equation}
	where $\widetilde{\kappa}_1 = \kappa_1 - \kappa_U>0$.
	
	In the meanwhile, for any $\cm{W}\in\widehat{\mathcal{S}}_{K}$ such that $\|\cm{W}-\cm{W}^*\|_\mathrm{F}\leq R$, where $R$ can be specified later to satisfy some conditions,
	\begin{align*}
		&\hspace{5mm}\left|\bm{h}_n(\cm{W}) - \bm{h}_n(\cm{W}^*)\right|\\
		&= \left|\frac{1}{n}\sum_{i=1}^{n}\phi^{\prime\prime}(\langle\cm{Z}^i,\cm{W}\rangle)\mathbb{I}_A\bm{z}^i\bm{z}^{i\prime} - \frac{1}{n}\sum_{i=1}^{n}\phi^{\prime\prime}(\langle\cm{Z}^i,\cm{W}^*\rangle)\mathbb{I}_A\bm{z}^i\bm{z}^{i\prime}\right|\\
		&\leq
		\frac{1}{n}\sum_{i=1}^{n}\left|\phi^{\prime\prime}(\langle\cm{Z}^i,\cm{W}\rangle)-\phi^{\prime\prime}(\langle\cm{Z}^i,\cm{W}^*\rangle)\right|\mathbb{I}_A\bm{z}^i\bm{z}^{i\prime}\\
		&=
		\frac{1}{n}\sum_{i=1}^{n}\left|\phi^{\prime\prime\prime}(\langle\cm{Z}^i,\bar{\cm{W}}\rangle)\langle\cm{Z}^i,\cm{W}-\cm{W}^*\rangle\right|\mathbb{I}_A\bm{z}^i\bm{z}^{i\prime},
	\end{align*}
	where $\bar{\cm{W}}$ lies between $\cm{W}$ and $\cm{W}^*$, and $\phi^{\prime\prime\prime}(z) = e^z(1-e^z)/(1+e^z)^3$. Given the event $A$ holds, choose $R<\tau$, we can lower bound the term,
	\begin{align*}
		|\langle\cm{Z}^i,\bar{\cm{W}}\rangle| &\geq  |\langle\cm{Z}^i,\cm{W}^*\rangle| - \sup_{\bm{\Delta}\in\widehat{\mathcal{S}}_{2K},\|\bm{\Delta}\|_\mathrm{F}\leq R}|\langle\cm{Z}^i,\bm{\Delta}\rangle|\\
		&\geq (\tau - R)\sup_{\bm{\Delta}\in\mathcal{S}_{2K}}|\langle\cm{Z}^i,\bm{\Delta}\rangle|.
	\end{align*}
	Notice that, for all $z\in\mathbb{R}$, the third order derivative of the function $\phi(z)$ is upper bounded as $|\phi^{\prime\prime\prime}(z)|\leq 1/|z|$. This relationship helps us further bound the term,
	\begin{align*}
		&\left|\phi^{\prime\prime\prime}(\langle\cm{Z}^i,\bar{\cm{W}}\rangle)\langle\cm{Z}^i,\cm{W}-\cm{W}^*\rangle\right|
		\leq
		\frac{|\langle\cm{Z}^i,\cm{W}-\cm{W}^*\rangle|}{|\langle\cm{Z}^i,\bar{\cm{W}}\rangle|}\\
		&\hspace{5mm}\leq \frac{\sup_{\bm{\Delta}\in\widehat{\mathcal{S}}_{2K},\|\bm{\Delta}\|_\mathrm{F}\leq R}|\langle\cm{Z}^i,\bm{\Delta}\rangle|}{(\tau - R)\sup_{\bm{\Delta}\in\mathcal{S}_{2K}}|\langle\cm{Z}^i,\bm{\Delta}\rangle|}\\
		&\hspace{5mm}\leq
		\frac{R\sup_{\bm{\Delta}\in\mathcal{S}_{2K}}|\langle\cm{Z}^i,\bm{\Delta}\rangle|}{(\tau - R)\sup_{\bm{\Delta}\in\mathcal{S}_{2K}}|\langle\cm{Z}^i,\bm{\Delta}\rangle|}
		=
		\frac{R}{\tau - R}.
	\end{align*}
	Hence, we can show that
	\begin{align*}
		&\mathbb{P}\left\{\sup_{\bm{\Delta}\in\mathcal{S}_{2K}}\left|\bm{\Delta}^\prime[\bm{h}_n(\cm{W})-\bm{h}_n(\cm{W}^*)]\bm{\Delta}\right|\geq t\right\}\\
		&\leq
		\mathbb{P}\left\{\frac{R}{\tau - R}\sup_{\bm{\Delta}\in\mathcal{S}_{2K}}\frac{1}{n}\sum_{i=1}^{n}\bm{\Delta}^\prime\bm{z}^i\bm{z}^{i\prime}\bm{\Delta}\geq t\right\}.
	\end{align*}
	By setting $t = \alpha_\text{RSM}R/(\tau-R)$, where $\alpha_\text{RSM} = 3\kappa_U/2$, we can use the equation (16) in Lemma \ref{lemma:RSC} to obtain, as long as $n\gtrsim d_{\mathcal{M}}$,
	\begin{align*}
		&\mathbb{P}\left\{\sup_{\bm{\Delta}\in\mathcal{S}_{2K}}\left|\bm{\Delta}^\prime[\bm{h}_n(\cm{W})-\bm{h}_n(\cm{W}^*)]\bm{\Delta}\right|\geq \frac{\alpha_\text{RSM}R}{\tau-R}\right\}\\
		&\leq 2\exp\left\{-c_Hn + 9d_{\mathcal{M}}\right\}.
	\end{align*}
	By rearranging terms, this is equivalent to
	\begin{align*}
		&\mathbb{P}\left\{\underbrace{\inf_{\bm{\Delta}\in\mathcal{S}_{2K}}\bm{\Delta}^\prime\bm{h}_n(\cm{W})\bm{\Delta}\leq\sup_{\bm{\Delta}\in\mathcal{S}_{2K}}\bm{\Delta}^\prime\bm{h}_n(\cm{W}^*)\bm{\Delta}-\frac{\alpha_\text{RSM}R}{\tau-R}}_{\text{denoted as the event }B_1}\right\}\\
		&\hspace{5mm}\leq2\exp\left\{-c_Hn + 9d_{\mathcal{M}}\right\}.
	\end{align*}
	If we define the event $B_2 = \{\inf_{\bm{\Delta}\in\mathcal{S}_{2K}}\bm{\Delta}^\prime\bm{h}_n(\cm{W}^*)\bm{\Delta} \leq \widetilde{\kappa}_1\}$ and denote its complementary event by $B_2^c$, and from (\ref{eq:h-true}), we know that $\mathbb{P}(B_2) \leq 2\exp\left\{-0.25cn+ 9d_{\mathcal{M}}\right\}$. It can be seen that
	\begin{align*}
		&\mathbb{P}\left(\left\{\inf_{\bm{\Delta}\in\mathcal{S}_{2K}}\bm{\Delta}^\prime\bm{h}_n(\cm{W})\bm{\Delta}\leq\widetilde{\kappa}_1-\frac{\alpha_\text{RSM}R}{\tau-R}\right\}\cap B_2^c\right)\\
		&\leq
		\mathbb{P}(B_1\cap B_2^c) \leq \mathbb{P}(B_1)
	\end{align*}
	
	So, if we choose $R$ to be sufficiently small, such that $ \alpha_\text{RSM}R/(\tau-R)\leq\widetilde{\kappa}_1/2$, it holds that,
	\begin{align*}
		&\mathbb{P}\left\{\inf_{\bm{\Delta}\in\mathcal{S}_{2K}}\bm{\Delta}^\prime\bm{h}_n(\cm{W})\bm{\Delta}\leq \frac{\widetilde{\kappa}_1}{2}\right\}\\
		&\leq \mathbb{P}\left(\left\{\inf_{\bm{\Delta}\in\mathcal{S}_{2K}}\bm{\Delta}^\prime\bm{h}_n(\cm{W})\bm{\Delta}\leq\widetilde{\kappa}_1-\frac{\alpha_\text{RSM}R}{\tau-R}\right\}\cap B_2^c\right)+\mathbb{P}(B_2).
	\end{align*}
	This, together with $\bm{H}_n(\cdot)\geq\bm{h}_n(\cdot)$, leads us to conclude that, when $n\gtrsim d_{\mathcal{M}}$, there exists some $R>0$, such that for any $\cm{W}\in\widehat{\mathcal{S}}_{2K}$ satisfying $\|\cm{W} - \cm{W}^*\|_\mathrm{F}\leq R$,
	\[
	\inf_{\bm{\Delta}\in\mathcal{S}_{2K}}\bm{\Delta}^\prime\bm{H}_n(\cm{W})\bm{\Delta}\geq \frac{\widetilde{\kappa}_1}{2},
	\]
	holds with probability 
	\[
	1-2\exp\left\{-0.25c_Hn + 9d_{\mathcal{M}}\right\}-2\exp\left\{-0.25cn+ 9d_{\mathcal{M}}\right\}.
	\]
	We accomplished our proof of the LRSC of $\mathcal{L}_n(\cm{W})$ around $\cm{W}^*$.
\end{proof}

\begin{lemma}\label{lem:sub-exp}
	For two sub-gaussian random variables, $x$ and $y$, when $\mathbb{E}(xy) = \mathbb{E}x\mathbb{E}y$, i.e. $x$ is independent to $y$, it holds that
	\[
	\|xy\|_{\psi_1}\leq \|x\|_{\psi_2}\|y\|_{\psi_2}.
	\]
\end{lemma}

\begin{lemma}\label{lem:sug-gaussian}
	\[
	\|\phi^\prime(\eta^i) - y^i\|_{\psi_2} \leq 0.25.
	\]
\end{lemma}
\begin{proof}
	Firstly, we observe that
	\begin{align*}
		\mathbb{E}\left(\exp\{\lambda\left[\phi^\prime(\eta^i) - y^i\right]\}|\bm{z}^i\right)
		&=
		\exp\{\lambda\phi^\prime(\eta^i)\}\exp\{-\phi(\eta^i) \}+ \{\lambda\left[\phi^\prime(\eta^i) - 1\right]\}\exp\{\eta^i -\phi(\eta^i) \}	\\
		&=
		\exp\{\lambda\phi^\prime(\eta^i) -\phi(\eta^i) \} [1+\exp\{\eta^i-\lambda\}]\\
		&= \exp\{\phi(\eta^i-\lambda) -\phi(\eta^i)+\lambda\phi^\prime(\eta^i)\}\\
		&= \exp\{0.5\lambda^2\phi^{\prime\prime}(\eta^*)\} \leq \exp\{0.125\lambda^2\}.
	\end{align*}
	It then holds that $	\mathbb{E}\left(\exp\{\lambda\left[\phi^\prime(\eta^i) - y^i\right]\}\right) = 	\mathbb{E}\{\mathbb{E}\left(\exp\{\lambda\left[\phi^\prime(\eta^i) - y^i\right]\}|\bm{z}^i\right)\}\leq\exp\{0.125\lambda^2\}$, and this implies that $\|\phi^\prime(\eta^i) - y^i\|_{\psi_2}\leq 0.25$.
\end{proof}

\begin{lemma} \label{lem:h}
	Under Assumption \ref{assum}, there exists a universal constant $\tau>0$ such that $\bm{h}(\cm{W}^*)\geq\kappa_1\bm{I}$, where $\kappa_1$ is a positive constant.
\end{lemma}
\begin{proof}
	We first show that for any $p_0\in(0,1)$, there exists a constant $\tau$ such that
	\[
	\mathbb{P}(|\langle\cm{W}^*,\cm{Z}^i\rangle|>\tau\sup_{\bm{\Delta}\in \mathcal{S}_{2K}}|\langle\bm{\Delta},\bm{z}^i\rangle|)\geq p_0.
	\] 
	
	We would separately show that
	\[
	\mathbb{P}(\underbrace{|\langle\cm{W}^*,\cm{Z}^i\rangle|>c_1\sqrt{d_{\mathcal{M}}}}_{\text{denoted by event $D_1$}})\geq \frac{p_0+1}{2}\]
	and
	\[
	\mathbb{P}(\underbrace{\sup_{\bm{\Delta}\in \mathcal{S}_{2K}}|\langle\bm{\Delta},\bm{z}^i\rangle|\leq c_2\sqrt{d_{\mathcal{M}}}}_{\text{denoted by event $D_2$}})\geq \frac{p_0+1}{2},
	\]  
	for some positive constants $c_1$ and $c_2$. Using the relationship $\mathbb{P}(D_1\cap D_2) = P(D_1)+P(D_2)-P(D_1^c\cup D_2^c)\geq  P(D_1)+P(D_2)-1$, it follows naturally that
	\begin{equation}\label{eq:indicator-event}
		\mathbb{P}(|\langle\cm{W}^*,\cm{Z}^i\rangle|>\tau\sup_{\bm{\Delta}\in \mathcal{S}_{2K}}|\langle\bm{\Delta},\bm{z}^i\rangle|)\geq \frac{p_0+1}{2}+\frac{p_0+1}{2}-1 = p_0,
	\end{equation}
	where $\tau = c_1/c_2$.
	
	Since $\bm{x}^i$ is a gaussian vector with mean zero and covariance $\bm{\Sigma}$, $\bm{z}^i = \bm{U}_G^\prime\bm{x}^i$ is a zero-mean gaussian vector with covariance given by $\bm{U}_G^\prime\bm{\Sigma}\bm{U}_G$, and $\langle\cm{W}^*,\cm{Z}^i\rangle = \langle\vectorize(\cm{W}^*),\bm{z}^i\rangle$ also follows a normal distribution with mean zero and variance (also its sub-Gaussian norm) upper bounded by $\kappa_U\|\cm{W}^*\|^2_\mathrm{F}$, where $\kappa_U = C_xC_u$.
	
	Since from Assumption \ref{assum}(C2), $\|\cm{W}^*\|^2_\mathrm{F}\geq \alpha\sqrt{d_{\mathcal{M}}}$, we can take $c_1$ to be sufficiently small such that
	\[
	\mathbb{P}(D_1) = 
	\mathbb{P}\left(\frac{|\langle\cm{W}^*,\cm{Z}^i\rangle|}{\kappa_U\|\cm{W}^*\|_\mathrm{F}}>\frac{c_1}{\kappa_U\alpha}\right)\geq \mathbb{P}(|x|>\frac{c_1}{\kappa_U\alpha})\geq\frac{p_0+1}{2},
	\]
	where $x$ is a gaussian variable with variance upper bounded by 1.
	
	Then, we can also observe that, for any fixed $\bm{\Delta}\in\mathcal{S}_{2K}$, $\langle\bm{\Delta},\bm{z}^i\rangle$ is a gaussian variable with zero mean and variance upper bounded by $\kappa_U$. We can use the concentration inequality for gaussian random variable to establish that
	\[
	\mathbb{P}\left(|\langle\bm{\Delta},\bm{z}^i\rangle|\geq t\right)\leq 2\exp(-\frac{t^2}{\kappa_U}),
	\]
	for all $t\in\mathbb{R}$. We can further use the union bound to show that
	\[
	\mathbb{P}\left(\sup_{\bm{\Delta}\in\mathcal{S}_{2K}}|\langle\bm{\Delta},\bm{z}^i\rangle|\geq t\right)
	\leq
	2\exp(-\frac{t^2}{\kappa_U}+7d_{\mathcal{M}}).
	\]
	Let $t = c_2\sqrt{d_{\mathcal{M}}}$ for some positive constant $c_2>\sqrt{7\kappa_U}$. We can choose $c_2$ large enough such that
	\[
	\mathbb{P}\left(\sup_{\bm{\Delta}\in\mathcal{S}_{2K}}|\langle\bm{\Delta},\bm{z}^i\rangle|\geq c_2\sqrt{d_{\mathcal{M}}}\right)
	\leq
	\frac{1-p_0}{2}.
	\]
	The probability at (\ref{eq:indicator-event}) is hence shown.
	
	Now, we will take a look at the matrix $\bm{h}(\cm{W}^*)$ and show that it is positive definite. Same as in Lemma \ref{lem:LRSC}, we denote the event $A = \{|\langle\cm{W}^*,\cm{Z}^i\rangle|>\tau\sup_{\bm{\Delta}\in \mathcal{S}_{2K}}|\langle\bm{\Delta},\bm{z}^i\rangle|\}$, and its complement by $A^c$, then
	\begin{align*}
		&\bm{\Delta}^\prime \bm{h}(\cm{W}^*) \bm{\Delta} = \frac{1}{n}\mathbb{E}\left\{\sum_{i=1}^{n}\phi^{\prime\prime}(\langle\cm{Z}^i,\cm{W}^*\rangle)\mathbb{I}_A\bm{\Delta}^\prime\bm{z}^i\bm{z}^{i\prime}\bm{\Delta}\right\}\\
		&=
		\frac{1}{n}\mathbb{E}\left\{\sum_{i=1}^{n}\phi^{\prime\prime}(\langle\cm{Z}^i,\cm{W}^*\rangle)\bm{\Delta}^\prime\bm{z}^i\bm{z}^{i\prime}\bm{\Delta}\right\} \\
		&- \frac{1}{n}\mathbb{E}\left\{\sum_{i=1}^{n}\phi^{\prime\prime}(\langle\cm{Z}^i,\cm{W}^*\rangle)\mathbb{I}_{A^c}\bm{\Delta}^\prime\bm{z}^i\bm{z}^{i\prime}\bm{\Delta}\right\}\\
		&= 
		\bm{\Delta}^\prime\bm{H}(\cm{W}^*)\bm{\Delta} - \frac{1}{n}\mathbb{E}\left\{\sum_{i=1}^{n}\phi^{\prime\prime}(\langle\cm{Z}^i,\cm{W}^*\rangle)\mathbb{I}_{A^c}\bm{\Delta}^\prime\bm{z}^i\bm{z}^{i\prime}\bm{\Delta}\right\}\\
		&\overset{\text{(i)}}{\geq} 
		\kappa_0 - \frac{1}{n}\sqrt{\mathbb{E}\left\{\sum_{i=1}^{n}[\phi^{\prime\prime}(\langle\cm{Z}^i,\cm{W}^*\rangle)]^2\langle\bm{\Delta}, \bm{z}^i\rangle^4\right\}}\sqrt{\mathbb{E}\left\{\sum_{i=1}^{n}\mathbb{I}_{A^c}\right\}}\\
		&\overset{\text{(ii)}}{\geq} 
		\kappa_0 - \frac{\sqrt{3\kappa_U}}{4}(1-p_0),
	\end{align*}
	where (i) follows from Assumption \ref{assum}(C2), And since $\langle\bm{\Delta}, \bm{z}^i\rangle$ is a gaussian variable with mean zero and variance bounded by $\kappa_U$, its fourth moment is bounded by $3\kappa_U$. Also, by $\phi^{\prime\prime}(z)\in(0,0.25)$ for all $z\in\mathbb{R}$, (ii) can be shown.
	
	Here, we can take $p_0$ to be small enough such that $0.25{\sqrt{3\kappa_U}}(1-p_0)\leq 0.5\kappa_0$ holds. It follows that
	\[
	\bm{\Delta}^\prime \bm{h}(\cm{W}^*) \bm{\Delta} \geq \kappa_1,
	\]
	with $\kappa_1 = 0.5\kappa_0$. We hence accomplished the proof of the lemma.
\end{proof}

\end{document}